\documentclass{article}

\usepackage{microtype}
\usepackage{graphicx}
\usepackage{subfigure}
\usepackage{booktabs} 

\usepackage{hyperref}

\usepackage[accepted]{icml2019}
\usepackage{amsmath,amsthm,amssymb}
\newtheorem{theorem}{Theorem}
\newtheorem{lemma}{Lemma}

\usepackage{algpseudocode}

\newcommand{\R}{\mathbb{R}}
\newcommand{\ol}{\mathcal{A}}
\newcommand{\w}{\mathring{w}}
\renewcommand{\v}{\mathring{v}}
\newcommand{\wealth}{\text{Wealth}}
\newcommand{\argmin}{\mathop{\text{argmin}}}
\newcommand{\clip}{\text{clip}}
\newcommand{\sign}{\text{sign}}

\icmltitlerunning{Matrix-Free Preconditioning in Online Learning}

\begin{document}

\twocolumn[
\icmltitle{Matrix-Free Preconditioning in Online Learning}



\icmlsetsymbol{equal}{*}

\begin{icmlauthorlist}
\icmlauthor{Ashok Cutkosky}{google}
\icmlauthor{Tamas Sarlos}{google}
\end{icmlauthorlist}

\icmlaffiliation{google}{Google Research, California, USA}

\icmlcorrespondingauthor{Ashok Cutkosky}{cutkosky@google.com}

\icmlkeywords{Online Learning, Stochastic Optimization}

\vskip 0.3in
]



\printAffiliationsAndNotice{}  

\begin{abstract}
We provide an online convex optimization algorithm with regret that interpolates between the regret of an algorithm using an optimal preconditioning matrix and one using a diagonal preconditioning matrix. Our regret bound is never worse than that obtained by diagonal preconditioning, and in certain setting even surpasses that of algorithms with full-matrix preconditioning. Importantly, our algorithm runs in the same time and space complexity as online gradient descent. Along the way we incorporate new techniques that mildly streamline and improve logarithmic factors in prior regret analyses. We conclude by benchmarking our algorithm on synthetic data and deep learning tasks.
\end{abstract}

\section{Online Learning}
This paper considers the online linear optimization (OLO) problem. An OLO algorithm chooses output vectors $w_t\in \R^d$ in response to linear losses $\ell_t(w)=g_t\cdot w$ for some $g_t\in \R^d$. Performance is measured by the \emph{regret} \citep{shalev2012online, zinkevich2003online}:
\[
R_T(\w) = \sum_{t=1}^T \ell_t(w_t)-\ell_t(\w) =\sum_{t=1}^T g_t\cdot (w_t - \w)
\]
OLO algorithms are important because they also solve solve online \emph{convex} optimization problems, in which the losses $\ell_t$ need be only convex by virtue of taking $g_t$ to be a gradient (or subgradient) of the $t$th convex loss. Even better, these algorithms also solve \emph{stochastic} convex optimization problems by setting $\ell_t$ to be the $t$th minibatch loss and $\w$ to be the global minimizer \citep{cesa2004generalization}. Due to both the simplicity of the linear setting and the power of the resulting algorithms, OLO has become a successful and popular framework for designing and analyzing many of the algorithms used to train machine learning models today.

Our goal is to obtain \emph{adaptive} regret bounds so that $R_T(\w)$ may be much smaller in easier problems while still maintaining optimal worst-case guarantees. One relevant prior result is the ``diagonal preconditioner'' approach of Adagrad-style algorithms \citep{duchi2011adaptive,mcmahan2010adaptive}:
\begin{align}
R_T(\w) \le \sum_{i=1}^d \|\w\|_\infty \sqrt{\sum_{t=1}^Tg_{t,i}^2}\label{eqn:diagada}
\end{align}
where $g_{t,i}$ indicates the $i$th coordinate of $g_t$. This bound can be achieved via gradient descent with learning rates properly tuned to the value of $\|\w\|_\infty$, and algorithms of this flavor have found much use in practice. Similar regret bounds that do not require tuning to the value of $\|\w\|_\infty$  can be obtained by making use of a Lipschitz assumption $\|g_t\|\le G$, leading to a bound of the form \citep{cutkosky2018black}:
\begin{align}
R_T(\w) \le O\left(\epsilon + \sum_{i=1}^d |\w_i|\sqrt{\log(\tfrac{d|\w_i|T}{\epsilon})\sum_{t=1}^T g_{t,i}^2}\right)\label{eqn:diagparamfree}
\end{align}
where $\epsilon$ is a free parameter representing the ``regret at the origin''. The extra logarithmic factor is an unavoidable penalty for this extra adaptivity \citep{mcmahan2012no}. This bound has the advantage that by Cauchy-Schwarz it is at most a logarithmic factor away from $\|\w\|_2\sqrt{\sum_{t=1}^T \|g_t\|_2^2}$, while the diagonal Adagrad bound may be a factor of $\sqrt{d}$ worse due to the $\|\w\|_\infty$ instead of $\|\w\|_2$.

Another type of bound is the ``full-matrix Adagrad'' bound \citep{duchi2011adaptive}
\begin{align}
R_T(\w) \le \|\w\|_2\text{Tr}\left(\sqrt{\sum_{t=1}^T g_tg_t^T}\right)\label{eqn:fullmatrixada}
\end{align}
or the more recent improved full-matrix bounds of \citep{cutkosky2018black, koren2017affine}:
\begin{align}
R_T(\w) \le \sqrt{d\sum_{t=1}^T \langle g_t, \w\rangle^2}\label{eqn:fullmatrixcoin}
\end{align}

The above bound (\ref{eqn:fullmatrixcoin}) may be much better than the diagonal bound, but unfortunately the algorithms involve manipulating a $d\times d$ matrix (often called a ``preconditioning matrix'') and so require $O(d^2)$ time per update. Our goal is to design an algorithm that maintains $O(d)$ time per update, but still manages to smoothly interpolate between the diagonal bound (\ref{eqn:diagparamfree}) and full-matrix bounds (\ref{eqn:fullmatrixcoin}).

Efficiently approximating the performance of full-matrix algorithms is an active area of research. Prior approaches include approximations based on sketching, low-rank approximations, and exploiting some assumed structure in the gradients \citep{luo2016efficient, gupta2018shampoo, agarwal2018case, gonen2015faster, martens2015optimizing}. The typical approach is to trade off computational complexity for approximation quality by maintaining some kind of lossy compressed representation of the preconditioning matrix. The properties of these tradeoffs vary: for some strategies one may obtain worse-regret than a non-full matrix algorithm (or even linear regret) if the data is particularly adversarial, while for others one may be unable to see nontrivial gains without significant complexity penalties. Our techniques are rather different, and so we make no complexity tradeoffs, never suffer worse regret than diagonal algorithms, and yet still obtain full-matrix bounds in favorable situations.

This paper is organized as follows: In Section \ref{sec:betting} we give some background in online learning analysis techniques that we will be using. In Sections \ref{sec:recursive}-\ref{sec:full_regret_bound} we state and analyze our algorithm that can interpolate between full-matrix and diagonal regret bounds efficiently. In Section \ref{sec:experiments} we provide an empirical evaluation of our algorithm.

\section{Betting Algorithms}\label{sec:betting}

A recent technique for designing online algorithms is via the wealth-regret duality approach \citep{mcmahan2014unconstrained} and betting algorithms \citep{orabona2016coin}. In betting algorithms, one keeps track of the ``wealth'':
\[
\wealth_T = \epsilon-\sum_{t=1}^T g_t\cdot w_t
\]
where $\epsilon>0$ is some user-defined hyperparameter. The goal is to make the wealth as big as possible, because
\[
R_T(\w) = \epsilon-\wealth_T - \sum_{t=1}^T g_t\cdot \w
\]
and in some sense the wealth is the only part of the above expression that the algorithm actually has control over.

Specifically, we want to obtain a statement like:
\[
\wealth_T \ge f\left(-\sum_{t=1}^T g_t\cdot \frac{\w}{\|\w\|}\right)
\]
for some function $f$, which exists only for analysis purposes here. Given this inequality, we can write:
\begin{align*}
    R_T(\w) &= \epsilon -\sum_{T=1}^T g_t\cdot \w - \wealth_T\\
    &\le \epsilon -\sum_{T=1}^T g_t\cdot \w - f\left(-\sum_{t=1}^T g_t\cdot \frac{\w}{\|\w\|}\right)\\
    &\le \epsilon + \sup_{X\in \R}  X \|\w\| - f(X)\\
    &= \epsilon + f^\star(\|\w\|)
\end{align*}
where $f^\star$ is the Fenchel conjugate, defined by $f^\star(u) = \sup_{x} x u - f(x)$. Formally, we have:

\begin{lemma}\label{thm:duality}
If $\wealth_T \ge f\left(-\sum_{t=1}^T \frac{g_t\cdot \w}{\|\w\|}\right) $ for some arbitrary norm $\|\cdot\|$ and function $f$, then $R_T(\w) \le \epsilon + f^\star(\|\w\|)$.
\end{lemma}

One way to increase the wealth is to view the vectors $w_t$ as some kind of ``bet'' and $g_t$ as some kind of outcome (e.g. imagine that $w_t$ is a portfolio and $-g_t$ is a vector of returns). Then the amount of ``money'' you win at time $t$ is $-g_t\cdot w_t$ and so $\wealth_T$ is the total amount of money you have at time $T$, assuming you started out with $\epsilon$ units of currency.

In order to leverage this metaphor, we make a critical assumption: $\|g_t\|_\star\le 1$ for all $t$.
Here $\|\cdot\|_\star$ is the dual norm, $\|g\|_\star=\sup_{\|w\|\le 1} g\cdot w$ (e.g. when $\|\cdot\|$ is the 2-norm, $\|\cdot\|_\star$ is also the 2-norm, and when $\|\cdot\|$ is the infinity-norm, $\|\cdot\|_\star$ is the 1-norm). There is nothing special about $1$ here; we may choose any constant, but use $1$ for simplicity.

Under this assumption, guaranteeing $R_T(0)\le \epsilon$ is equivalent to never going into debt (i.e. $\wealth_T < 0$). We assure this by never betting more than we have: $\|w_t\|\le \wealth_{t-1}$. In fact, in order to simplify subsequent calculations, we will ask for a somewhat stronger restriction:
\begin{equation}
\|w_t\|< \frac{1}{2}\wealth_{t-1}\label{eqn:wlessthanwealth}
\end{equation}

Given (\ref{eqn:wlessthanwealth}), we can also write
\[
w_t = v_t \wealth_{t-1}
\]
where $v_t$ is a vector with $\|v_t\|<1/2$, which we call the ``betting fraction''. $v_t$ is a kind of ``learning rate'' analogue. However, in the $d$-dimensional setting $v_t$ is only a $d$-dimensional vector, while previous full-matrix algorithms use a learning rate analogue that is a $d\times d$ matrix.

\subsection{Constant $v_t$}
To understand the potential of this approach, consider the case of a fixed betting fraction $v_t=\v$. Using the inequality $\log(1-x)\ge -x-x^2$ for all $x\le 1/2$, we proceed:

\begin{align}
    \wealth_T&=\epsilon\prod_{t=1}^T(1-g_t\cdot \v)\nonumber\\
    \log(\wealth_T) &= \log(\epsilon)+ \sum_{t=1}^T \log(1-g_t\cdot \v)\label{eqn:logfixedwealth}\\
    &\ge\log(\epsilon) + \sum_{t=1}^T -g_t\cdot \v - (g_t\cdot \v)^2\label{eqn:logfixedwealthbound}
\end{align}
Now if we set 
\begin{align*}
    \v=-\frac{\w}{\|\w\|}\frac{\sum_{t=1}^T g_t\cdot \frac{\w}{\|\w\|}}{2\left|\sum_{t=1}^T g_t\cdot \frac{\w}{\|\w\|}\right| + 2\sum_{t=1}^T ( g_t\cdot \frac{\w}{\|\w\|})^2}
\end{align*}
we obtain:
\begin{align*}
 \wealth_T&\ge \epsilon\exp\left[\frac{\left(\sum_{t=1}^T g_t\cdot \frac{\w}{\|\w\|}\right)^2}{4\left|\sum_{t=1}^T g_t\cdot \frac{\w}{\|\w\|}\right| + 4\sum_{t=1}^T ( g_t\cdot \frac{\w}{\|\w\|})^2}\right]\\
 &= f\left(-\sum_{t=1}^T g_t\cdot \frac{\w}{\|\w\|}\right)
\end{align*}

Where $f(x) = \epsilon\exp\left[\tfrac{x^2}{4|x| + 4\sum_{t=1}^T ( g_t\cdot \tfrac{\w}{\|\w\|})^2}\right]$.
Finally, bound $f^\star$ by Lemma 19 of \cite{cutkosky2018black}:
\[
R_T(\w) \le \epsilon + f^\star(\|\w\|) \le \tilde O\left[\epsilon+\sqrt{\sum_{t=1}^T (g_t\cdot \w)^2}\right]
\]

This is actually a factor of up to $\sqrt{d}$ better than the full-matrix guarantee (\ref{eqn:fullmatrixcoin}) and, more importantly, there are no matrices in this algorithm! Instead, the role of the preconditioner is played by the vector $\v$, which corresponds to a kind of ``optimal direction''.



\subsection{Varying $v_t$}
Now that we know that there exists a good fixed betting fraction $v$ given oracle tuning, we turn to the problem of using varying $v_t$. To do this we use the reduction developed by \citet{cutkosky2018black} for recasting the problem of choosing $v_t$ as itself an online learning problem. The first step is to calculate the wealth with changing $v_t$:

\begin{align}
    \log(\wealth_T) = \log(\epsilon) + \sum_{t=1}^T \log(1-g_t\cdot v_t)\label{eqn:logwealth}
\end{align}

Next, denote the wealth of the algorithm that uses a fixed fraction $\v$ as $\wealth_T(\v)$ and then subtract (\ref{eqn:logwealth}) from (\ref{eqn:logfixedwealth}):
\begin{align*}
    &\log(\wealth_T(\v))-\log(\wealth_T)\\
    &\quad=\sum_{t=1}^T \log(1-g_t\cdot \v) - \log(1-g_t\cdot v_t)\\
    &\quad=\sum_{t=1}^T \ell_t(v_t) -\ell_t(\v)
\end{align*}
where we define $\ell_t(x) = -\log(1-g_t\cdot x)$. Notice that $\ell_t$ is convex, so we can try to find $\v$ by using an online convex optimization algorithm that outputs $v_t$ in response to the loss $\ell_t$. Let $R^v_T$ be the regret of this algorithm. Then by definition of regret, for any $\v$:
\begin{align*}
    \log(\wealth_T(\v)) - \log(\wealth_T)= R^v_T(\v)
\end{align*}
Combining the above with inequality (\ref{eqn:logfixedwealthbound}) we have
\begin{align}
    \log(\wealth_T) &=  \log(\wealth_T(\v)) -R^v_T(\v)\nonumber\\
    &\mkern-54mu= \log(\epsilon)+\sum_{t=1}^T -g_t\cdot \v - (g_t\cdot \v)^2 - R^v_T(\v)\label{eqn:regretrecursion}
\end{align}
So now need to find a $\v$ that maximizes this expression.

Our analysis diverges from prior work at this point. Previously, \citep{cutkosky2018black} observed that $\ell_t$ is exp-concave, and so by using the Online Newton Step \citep{hazan2007logarithmic} algorithm one can make $R^v_T(\v)=O(\log(T))$ and obtain regret
\begin{align}
R_T(\w) \le O\left(\|\w\|\sqrt{\log(\tfrac{\|\w\|T^{4.5}}{\epsilon})\sum_{t=1}^T \|g_t\|^2}\right) \label{eqn:ons}
\end{align}

Instead, we take a different strategy by using \emph{recursion}. The idea is simple: we can apply the exact same reduction we have just outlined to design an ``inner'' coin-betting strategy for choosing $v_t$ and minimizing $R^v_T$. The major subtlety that needs to be addressed is the restriction $\|v_t\|\le 1/2$. Fortunately, \cite{cutkosky2018black} also provides a black-box reduction that converts any unconstrained optimization algorithm into a constrained algorithm without modifying the regret bound, and so we can essentially ignore the constraint on $v_t$ in our analysis.

\section{Recursive Betting Algorithm}\label{sec:recursive}

The key advantage of using a recursive strategy to choose $v_t$ is that the regret $R^v_T(\v)$ may depend strongly on $\|\v\|$. Since in many cases $\|\v\|$ is small, this results in better overall performance than if we were to directly apply the Online Newton Step algorithm. We formalize this strategy and intuition in Algorithm \ref{alg:recurse} and Theorem \ref{thm:recursiveregret}.

\begin{algorithm}
\caption{Recursive Optimizer}\label{alg:recurse}
\begin{algorithmic}[1]
\Procedure{RecursiveOptimizer}{$\epsilon$}
\State $\wealth_0\gets \epsilon$.
\State Initialize \textsc{InnerOptimizer}.
\For{$t=1\dots T$}
\State Let $v_t$ be the $t$th output of \textsc{InnerOptimizer}.
\State $w_t\gets \wealth_{t-1}v_t$.
\State Output $w_t$, receive $g_t$.
\State $\wealth_t\gets \wealth_{t-1} - g_t\cdot w_t$.
\State $z_t\gets \frac{g_t}{1-g_t\cdot v_t}=\frac{d}{d v_t}-\log(1-g_t\cdot v_t)$.
\State Send $z_{t}$ as $t^{\text{th}}$ gradient to \textsc{InnerOptimizer}.
\EndFor
\EndProcedure
\end{algorithmic}
\end{algorithm}

\begin{theorem}\label{thm:recursiveregret}
Suppose $\|g_t\|_\star \le 1$ for some norm $\|\cdot\|$ for all $t$. Further suppose that \textsc{InnerOptimizer} satisfies $\|v_t\|\le 1/2$ and guarantees regret nearly linear in $\|\v\|$:
\begin{align*}
R^v_T(\v)=\sum_{t=1}^T z_t\cdot v_t-z_t\cdot \v  &\le \epsilon + \|\v\| G_T(\v/\|\v\|)
\end{align*}
for some function $G_T(\v/\|\v\|)$ for any $\v$ with $\|\v\|\le 1/2$.
Then if $-\sum_{t=1}^T g_t\cdot \frac{\w}{\|\w\|}\ge 2G_T(\w/\|\w\|)$, Algorithm \ref{alg:recurse} obtains
\[
R_T(\w) \le \tilde O\left(\epsilon + \sqrt{\sum_{t=1}^T (g_t\cdot \w)^2}\right)
\]
and otherwise
\[
R_T(\w) \le \epsilon + 2\|\w\|G_T(\w/\|\w\|)
\]
\end{theorem}

Let us unpack the condition $-\sum_{t=1}^T g_t\cdot \frac{\w}{\|\w\|}\ge 2G_T(\w/\|\w\|)$. First we consider the LHS. Observe that $-\sum_{t=1}^T \frac{\w}{\|\w\|}\cdot g_t$ is the regret at $\w/\|\w\|$ of an algorithm that always predicts $0$. In a classic adversarial problem we should expect this value to grow as $\Omega(T)$. Even in the case that each $g_t$ is an i.i.d.\ mean-zero random variable, we should expect growth of at least $\Omega(\sqrt{T})$. For the RHS, observe that so long as \textsc{InnerOptimizer} obtains the optimal $T$-dependence in its regret bound, we should expect $G_T=\tilde O(\sqrt{T})$ - for example the algorithm of \citep{cutkosky2018black} obtains $G_T(\v/\|\v\|) = O\left(\sqrt{\sum_{t=1}^T \|g_t\|_\star^2 \log(T^{4.5}/\epsilon)}\right)$ for any $\|\v\|\le 1/2$. Therefore the condition $-\sum_{t=1}^T g_t\cdot \frac{\w}{\|\w\|}\ge 2G_T(\w/\|\w\|)$ can be viewed as saying that the $g_t$ in some sense violate standard concentration inequalities and so are clearly not mean-zero random variables: intuitively, there is some amount of signal in the gradients.

As a simple concrete example, suppose the $g_t$ are i.i.d.\ random vectors with covariance $\Sigma$ and mean $-\sqrt{\epsilon}x$, where $x$ is the eigenvector of $\Sigma$ with smallest eigenvalue. Then $\sum_{t=1}^T g_t\cdot x$ will grow as $\Theta(T\sqrt{\epsilon})$, and so for sufficiently large $T$ we will obtain the full-matrix regret bound where $R_T(x)$ grows with $\sum_{t=1}^T (g_t\cdot x)^2$. This has expectation $x^T \Sigma xT+\epsilon T = (\lambda_d+\epsilon)T$, where $\lambda_d$ is the smallest eigenvalue of $\Sigma$. In contrast, a standard regret bound may depend on $\sum_{t=1}^T\|g_t\|_2^2$. This has expectation $\text{Trace}(\Sigma)T+\epsilon T$, which is a factor of $d$ larger for small $\epsilon$, and even more if $\Sigma$ is poorly conditioned.

Next, let us consider the second case in which the regret bound is $O(\epsilon + \|\w\|G_T)$. This bound is also actually a subtle improvement on prior guarantees. For example, if \textsc{InnerOptimizer} guarantees regret $R^v_T(\v)\le \epsilon+\|\v\| \sqrt{\log(\|\v\|T/\epsilon)\sum_{t=1}^T \|g_t\|_\star^2}$, we can use the fact that $\|\v\|\le 1/2$ to bound $G_T$ by $\sqrt{\log(T/\epsilon)\sum_{t=1}^T \|g_t\|_\star^2}$. Thus the bound $\|\w\|G_T$ is better than previous regret bounds like (\ref{eqn:ons}) due to removing the $\|\w\|$ from inside the log.

In summary, we improve prior art in two important ways:
\begin{enumerate}
    \item When the sum of the gradients is greater than $\tilde \Omega(\sqrt{T})$, we obtain the optimal full-matrix regret bound.
    \item When the sum of the gradients is smaller, our regret bound grows only linearly with $\|\w\|$, without any $\sqrt{\log(\|\w\|)}$ factor.
\end{enumerate}

Both of these improvements appear to contradict lower bounds. First, \citep{luo2016efficient} suggests that the factor $\sqrt{d}$ is necessary in a full-matrix regret bound, which seems to rule out improvement 1. Second, \citep{mcmahan2014unconstrained, orabona2013dimension} state that a $\sqrt{\log(\|\w\|)}$ factor is required when $\|\w\|$ is unknown, appearing to rule out improvement 2. We are consistent with these results because of the condition $-\sum_{t=1}^T g_t\cdot \frac{\w}{\|\w\|}\ge 2G_T(\w/\|\w\|)$. Both lower bounds use $g_t$ whose coordinates are random $\pm1$. However, the bound of \citep{luo2016efficient} involves a ``typical sequence'', which concentrates appropriately about zero and does not satisfy the condition to have our improved full-matrix bound. In contrast, the bounds of \citep{mcmahan2014unconstrained, orabona2013dimension} are stated for 1-dimensional problems and rely on \emph{anti-concentration}, so that the adversarial sequence is very atypical and does satisfy the condition, yielding our full-matrix bound that does include the log factor.

In Section \ref{sec:inner} we propose a diagonal algorithm for use as \textsc{InnerOptimizer}. This will enable Algorithm \ref{alg:recurse} to interpolate between a diagonal regret bound and the full-matrix guarantee. At first glance, this phenomenon is somewhat curious: how can an algorithm that keeps only per-coordinate state manage to adapt to the covariance between pairs of coordinates? The answer lies in the gradients supplied to the \textsc{InnerOptimizer}: $\tfrac{g_t}{1-g_t\cdot v_t}$. The denominator of this expression actually contains information from all coordinates, and so even when \textsc{InnerOptimizer} is a diagonal algorithm it still has access to interactions between coordinates.

Now we sketch a proof of Theorem \ref{thm:recursiveregret}. We will drop constants, logs and $\epsilon$ and leave full details to Appendix \ref{sec:proofrecursive}.
\begin{proof}[Proof Sketch of Theorem \ref{thm:recursiveregret}]

We start from (\ref{eqn:regretrecursion}) and use our assumption on the regret bound of \textsc{InnerOptimizer}: 
\begin{align*}
    \log(\wealth_T) &\ge \sum_{t=1}^T -g_t\cdot \v - (g_t\cdot \v)^2 - \|\v\|G_T(\v/\|\v\|)
\end{align*}
for all $\v$. So now we choose $\v$ to optimizes the bound.

Let us suppose that $\v$ is of the form $\v= x\frac{\w}{\|\w\|}$ for some $x$ so that $\v/\|\v\|=\w/\|\w\|$. We consider two cases: either $-\sum_{t=1}^T g_t\cdot \tfrac{\w}{\|\w\|} \ge 2G_T(\w/\|\w\|)$ or not.

\noindent\textbf{Case 1 $-\sum_{t=1}^T g_t\cdot \w/\|\w\| \ge 2 G_T(\w/\|\w\|)$:}

In this case we have
\begin{align*}
-\sum_{t=1}^T g_t\cdot \v - \|\v\|G_T(\v/\|\v\|) \ge -\frac{1}{2}\sum_{t=1}^T g_t\cdot \v
\end{align*}
Therefore we have
\begin{align*}
    \log(\wealth_T) &\ge \sum_{t=1}^T -\frac{1}{2}g_t\cdot \v - (g_t\cdot \v)^2
\end{align*}

So now using essentially the same argument as in the fixed $v$ case, we end up with a full-matrix regret bound:
\[
R_T(\w) =\tilde O\left(\epsilon + \sqrt{\sum_{t=1}^T (g_t\cdot \w)^2}\right)
\]
\noindent\textbf{Case 2 $-\sum_{t=1}^T g_t\cdot \w/\|\w\| < 2 G_T(\w/\|\w\|)$}

In this case, observe that since we guarantee $\wealth_T>0$ no matter what strategy is used to pick $v_t$, we have
\begin{align*}
R_T(\w) &= \epsilon -\wealth_T -\sum_{t=1}^T g_t\cdot \w\\
&\le \epsilon + 2 \|\w\|G_T(\w/\|\w\|)
\end{align*}
And so we are done.
\end{proof}

\section{Diagonal \textsc{InnerOptimizer}}\label{sec:inner}

As a specific example of an algorithm that can be used as \textsc{InnerOptimizer}, we provide Algorithm \ref{alg:diag}. This algorithm will achieve a regret bound similar to (\ref{eqn:diagparamfree}). Here we use $\clip(x,a,b)$ to indicate truncating $x$ to the interval $[a,b]$. Algorithm \ref{alg:diag} works by simply applying a separate 1-dimensional optimizer on each coordinate 
of the problem. Each 1-dimensional optimizer is itself a coin-betting algorithm that uses Follow-the-Regularized leader \cite{hazan2016introduction} to choose the betting fractions $v_t$. There are also two important modifications at lines 7 and 12-15 that implement the unconstrained-to-constrained reduction.

\begin{algorithm}
\caption{Diagonal Betting Algorithm}\label{alg:diag}
\begin{algorithmic}[1]
\Procedure{DiagOptimizer}{$\epsilon$, $\eta$}
\State $\wealth_{0,i}\gets \epsilon$ for $i\in\{1,\dots,d\}$.
\State $A_{0,i}\gets 5$ and $v_{1,i}\gets 0$  for $i\in\{1,\dots,d\}$.
\For{$t=1\dots T$}
\For{$i=1\dots d$}
\State $x_{t,i}\gets v_{t,i}\wealth_{t-1,i}$.
\State Set $w_{t,i}=\clip(x_{t,i}, -1/2, 1/2)$.
\EndFor
\State Output $w_t = (w_{t,1},\dots,w_{t,d})$.
\State Receive $g_t$ with $g_{t,i} \in[-1,1]$ for all $i$.
\For{$i=1\dots d$}
\State $\tilde g_{t,i}\gets g_{t,i}$
\If{$g_{t,i} (x_{t,i}-w_{t,i})<0$}
\State $\tilde g_{t,i}\gets 0$.
\EndIf
\State $\wealth_{t,i}\gets \wealth_{t-1,i}-x_{t,i}\tilde g_{t,i}$.
\State $z_{t,i}\gets \frac{d}{dv_{t,i}} -\log(1-\tilde g_{t,i}v_{t,i})=\frac{\tilde g_{t,i}}{1-\tilde g_{t,i}v_{t,i}}$.
\State $A_{t,i}\gets A_{t-1,i} + z_{t,i}^2$.
\State $v_{t,i}\gets \clip\left(\frac{-2\eta \sum_{t'=1}^t z_{t',i}}{A_{t,i}},-1/2,1/2\right)$
\EndFor
\EndFor
\EndProcedure
\end{algorithmic}
\end{algorithm}
Before analyzing \textsc{DiagOptimizer}, we perform a second analysis of \textsc{RecursiveOptimizer} that makes no restrictions on \textsc{InnerOptimizer}. We will eventually see that Algorithm \ref{alg:diag} is essentially an instance of \textsc{RecursiveOptimizer} and so this Lemma will be key in our analysis:
\begin{lemma}\label{thm:generalrecursiveregret}
Suppose $\|g_t\|_\star \le 1$ for all $t$. Suppose \textsc{InnerOptimizer} satisfies $\|v_t\|\le 1/2$ and has regret $R^v_T(\v)$. Then \textsc{RecursiveOptimizer} obtains regret
\begin{align*}
    R_T(\w)&\le \inf_{c\in [0,1/2]}\epsilon + \frac{\|\w\|}{c}\left(\log\left(\frac{\|\w\|}{c\epsilon}\right)-1\right)\\
    &\quad\quad+ \|\w\|cZ + \frac{\|\w\|}{c}R^v_T\left(c\frac{\w}{\|\w\|}\right)
\end{align*}
where $Z=\sum_{t=1}^T \left(g_t\cdot \frac{\w}{\|\w\|}\right)^2$.
\end{lemma}
In words, we have written the regret of \textsc{RecursiveOptimizer} as a kind of tradeoff between $Z$, which is proportional to the quantity inside the square root of a full-matrix bound, and the regret of the \textsc{InnerOptimizer}. This makes it easier to compute the regret when \textsc{InnerOptimizer}'s regret bound does not satisfy the conditions of Theorem \ref{thm:recursiveregret}.

\begin{theorem}\label{thm:diaganalysis}
Suppose $\|g_t\|_\infty \le 1$ for all $t$. Then Algorithm \ref{alg:diag} guarantees regret $R_T(\w)$ at most:
\begin{align*}
    &d\epsilon +O\left(\sum_{i=1}^d|\w_i|\max\left\{\sqrt{\frac{G_i}{\eta}\log\left(\frac{|\w_i|G_i^{\eta}\sqrt{\frac{G_i}{\eta}}}{\epsilon}\right)}, \right.\right.\\
    &\quad\quad\quad\quad\left.\left.\log\left(\frac{|\w_i|G_i^{\eta}\sqrt{G_i/\eta}}{\epsilon}\right)\right\}\right)
\end{align*}
for all $\w$ with $\|\w\|_\infty \le 1/2$, where $G_i = \sum_{t=1}^T g_{t,i}^2$. Further, by using $\epsilon/d$ instead of $\epsilon$ and setting $\eta=1/2$, we can also re-write this as:  
\begin{align*}
    R_T(\w) & \le \epsilon + \|\w\|_\infty G_T(\w/\|\w\|_\infty),
\end{align*}
where 
\begin{align*}
    G_T(x)&= O\left(\sum_{i=1}^d|x_i|\max\left\{\sqrt{G_i\log\left(\frac{dZ_i}{\epsilon}\right)}, \right.\right.\\
    &\quad\quad\quad\quad\left.\left.\log\left(\frac{dG_i}{\epsilon}\right)\right\}\right)
\end{align*}
\end{theorem}
Let us briefly unpack this bound. Ignoring log factors to gain intuition, the bound is $\sum_{i=1}^d |\w_i|\sqrt{\sum_{t=1}^T g_{t,i}^2}$. Note that this improves upon the diagonal Adagrad bound (\ref{eqn:diagada}) by virtue of depending on each $|\w_i|$ rather than the norm $\|\w\|_\infty$, and by Cauchy-Schwarz it is bounded by $\|\w\|_2\sqrt{\sum_{t=1}^T \|g_t\|_2^2}$, which matches classic ``dimension-free'' bounds. Note however that this bound is \emph{not} strictly dimension-free due to the $d\epsilon$ term. Even setting $\epsilon=1/d$ will incur a $\log(d)$ penalty due to the $\log(1/\epsilon)$ factor. Most importantly, however, \textsc{DiagOptimizer} satisfies the conditions on \textsc{InnerOptimizer} in Theorem \ref{thm:recursiveregret}.

Theorem \ref{thm:diaganalysis} is also notable for its logarithmic factor, which can be made $O(\log(|\w_i|G_i^{\eta+1/2}/\epsilon\sqrt{\eta}))$ for any $\eta$. This is an improvement over prior bounds such as \citep{cutkosky2018black} in that the power of the $O(T)$ term $G_i$ inside the logarithm is smaller
However, the optimal value for this exponent is $1/2$ \citep{mcmahan2014unconstrained}, which this bound cannot obtain. Instead, we show in Appendix \ref{sec:optlog} that a simple doubling-trick scheme does allow us to obtain the optimal rate. To our knowledge this is the first time such a rate has been achieved: prior works achieve the optimal log factor, but have worse adaptivity to the values of $g_t$, depending on $T$ or $\sum_{t=1}^T |g_t|$ instead of $\sum_{t=1}^T g_t^2$ \citep{mcmahan2014unconstrained, orabona2014simultaneous}.

\begin{proof}[Proof Sketch of Theorem \ref{thm:diaganalysis}]
First, we observe that Algorithm \ref{alg:diag} is running $d$ copies of a 1-dimensional optimizer. Because we have
\begin{align*}
R_T(\w) = \sum_{t=1}^T g_t\cdot(w_t-\w)=\sum_{i=1}^d \sum_{t=1}^T g_{t,i}(w_{t,i} - \w_i)
\end{align*}
we may analyze each dimension individually and then sum the regrets. So let us focus on a single 1-dimensional optimizer, and drop all $i$ subscripts for simplicity.

Next, we address the truncation of $w_t$ and modifications to $g_t$. This is a 1-dimensional specialization of the unconstrained-to-constrained reduction of \citep{cutkosky2018black}. Let $g_t$ be the (original, unmodified) gradient, and let $\tilde g_t$ be the modified gradient (so $\tilde g_t = g_t$ or $\tilde g_t = 0$). A little calculation shows that
\[
\tilde g_t(x_t - \w) \ge g_t(w_t -\w)
\]
for any $\w\in [-1/2,1/2]$. Therefore, the regret $\sum_{t=1}^T g_t(w_t-\w)$ is upper-bounded by $\sum_{t=1}^T \tilde g_t (x_t -\w)$. This quantity is simply the regret of an algorithm that uses gradients $\tilde g_t$ and outputs $x_t$. Now we interpret $x_t$ as the predictions of a coin-betting algorithm that uses betting fractions $v_t$ in response to the gradients $\tilde g_t$. Thus we may analyze the regret of the $x_t$ with respect to the $\tilde g_t$ using coin-betting machinery. To this end, observe that $A_t = 5+ \sum_{t'=1}^t z_{t'}^2$, so that
\[
v_{t+1} = \argmin_{v\in [-1/2,1/2]} \sum_{t'=1}^t z_{t'} v + \tfrac{v^2}{4\eta}\left(5+\sum_{t'=1}^t z_{t'}^2\right)
\]
Since $z_t$ is the derivative of $\log(1-\tilde g_tv)$ evaluated at $v_t$, we see that $v_t$ are the outputs of a Follow-the-Regularized-Leader (FTRL) algorithm with regularizers $v^2\tfrac{1}{4\eta}(5+\sum_{t=1}^t z_{t'}^2)$. That is, we are actually using Algorithm \ref{alg:recurse} with \textsc{InnerOptimizer} equal to this FTRL algorithm. Using the FTRL analysis of \citep{mcmahan2017survey}, we then have
\[
R^v_T(\v) \le \tfrac{\v^2}{4\eta}\left(5+\sum_{t=1}^T z_{t}^2\right)+\sum_{t=1}^T\frac{\eta z_t^2}{5+\sum_{t'=1}^{t-1} z_{t'}^2}
\]
Next, by convexity we have $\log(a) + b/(a+4)\le \log(a+b)$ for all $a>0$, $0<b<4$. Since $|w_t|\le 1/2$ and $|g_t|\le 1$, $z_t^2\le 4$. Therefore by induction we can show:
\[
\sum_{t=1}^T\tfrac{z_t^2}{5+\sum_{t'=1}^{t-1} z_{t'}^2}\le \log\left(1+\sum_{t=1}^T z_t^2\right)
\]
so that since $|z_t|\le 2|\tilde g_t|\le 2|g_t|$,
\begin{align*}
R^v_T(\v) &\le &\le \tfrac{\v^2}{4\eta} \left(5+4\sum_{t=1}^T g_{t}^2\right) + \eta \log\left(1+4\sum_{t=1}^T g_t^2\right)
\end{align*}

Therefore by Lemma \ref{thm:generalrecursiveregret} we have
\begin{align*}
    R_T(\w) & \le \epsilon + \frac{|\w|}{c}(\log(|\w|/c\epsilon) -1) + |\w|c Z \\
    &\mkern-18mu+ \frac{|\w|c}{\eta} \left(5+4\sum_{t=1}^T g_{t}^2\right) + \frac{|\w|\eta }{c} \log\left(1+4\sum_{t=1}^T g_t^2\right)\\
    &= \epsilon + O\left(\frac{|\w|}{c}\log\left(\frac{|\w|G^{\eta}}{c\epsilon}\right) + \frac{|\w|cG}{\eta}\right)
\end{align*}
for all $c\in[0,1/2]$, where we have observed that $Z=\sum_{t=1}^T g_t^2=G$ in one dimension, and dropped various constants for simplicity. Optimizing for $c$ we have
\begin{align*}
    R_T(\w) & \le \epsilon +O\left(|\w|\max\left\{\sqrt{\frac{G}{\eta}\log\left(\frac{|\w|G^{\eta}\sqrt{\frac{Z}{\eta}}}{\epsilon}\right)}, \right.\right.\\
    &\quad\quad\quad\quad\left.\left.\log\left(\frac{|\w|G^{\eta}\sqrt{G/\eta}}{\epsilon}\right)\right\}\right)
\end{align*}
The statement for $G_T$ comes from simply observing that $|\w|_\infty\le 1/2$.
\end{proof}

\section{Full Regret Bound}\label{sec:full_regret_bound}
Now we combine the \textsc{DiagOptimizer} of the previous section with \textsc{RecursiveOptimizer}. There are only a few details to address. First, since the analysis of \textsc{DiagOptimizer} is specific to the infinity-norm, we set $\|\cdot\|$ to be the infinity-norm in \textsc{RecursiveOptimizer} and Theorem \ref{thm:recursiveregret}, which has $\|\cdot\|_1$ as dual norm. Second, note that the gradients provided to \textsc{InnerOptimizer} satisfy $\|z_t\|_\infty = \|g_t/(1-g_t\cdot v_t)\|_\infty \le 2\|g_t\|_\infty\le 2$ since $1\ge \|g_t\|_1\ge \|g_t\|_\infty$. Since the analysis of \textsc{DiagOptimizer} requires gradients bounded by 1, we rescale the gradients by a factor of 2, which scales up the regret by the same constant factor of 2. Therefore, we see that \textsc{DiagOptimizer} satisfies the hypotheses of Theorem \ref{thm:recursiveregret} with
\begin{align*}
    G_T(v/\|v\|_\infty) &= O\left(\sum_{i=1}^d\frac{|v_i|}{\|v\|_\infty}\max\left\{\sqrt{G_i\log\left(\frac{dG_i}{\epsilon}\right)}, \right.\right.\\
    &\quad\quad\quad\quad\left.\left.\log\left(\frac{dG_i}{\epsilon}\right)\right\}\right)
\end{align*}
Thus by Theorem \ref{thm:recursiveregret}, in all cases we have the diagonal bound:
\begin{align}
    R_T(\w) &\le \tilde O\left(\epsilon+\sum_{i=1}^d|\w_i|\sqrt{\sum_{t=1}^Tg_{t,i}^2}\right)\label{eqn:simplediag}
\end{align}
and whenever $-\sum_{t=1}^T  g_t\cdot \tfrac{\w}{\|\w\|_\infty}\ge 2G_T(\w/\|\w\|)=\tilde O(\sqrt{T})$, we have
\begin{align*}
    R_T(\w)&\le \tilde O\left(\epsilon + \sqrt{\sum_{t=1}^T (g_t\cdot \w)^2}\right)
\end{align*}
Note that this may be even a factor of $\sqrt{d}$ better than the standard full-matrix regret bounds (\ref{eqn:fullmatrixada}), (\ref{eqn:fullmatrixcoin}).

We recall that by Cauchy-Schwarz, the bound (\ref{eqn:simplediag}) also implies:
\begin{align*}
    R_T(\w) & \le \tilde O\left(\|\w\|_2\sqrt{\sum_{t=1}^T \|g_t\|_2^2}\right)
\end{align*}
which is the standard non-diagonal adaptive regret bound. Note that this may be a factor of $\sqrt{d}$ better than the diagonal Adagrad bound (\ref{eqn:diagada}) when both $\w$ and the gradients are dense.


\section{Experiments}\label{sec:experiments}
We implemented \textsc{RecursiveOptimizer} in TensorFlow \citep{abadi2016tensorflow} and ran benchmarks on both synthetic data as well as several deep learning tasks (see Appendix \ref{sec:experimentsappendix} for full details).\footnote{Code available at: \url{https://github.com/google-research/google-research/tree/master/recursive_optimizer}} We found that using the recent algorithm \textsc{ScInOL} of \citep{kempka2019adaptive} as the inner optimizer instead of Algorithm \ref{alg:diag} provided qualitatively similar results on synthetic data but better empirical performance on the deep learning tasks, so we report results using \textsc{ScInOL} as the inner optimizer. We stress that since \textsc{ScInOL} has essentially the same regret bound as in Theorem \ref{thm:diaganalysis} (with slightly worse log factors), this substitution maintains our theoretical guarantees while allowing us to inherit the scale-invariance properties of \textsc{ScInOL}. This actually highlights another advantage of our reduction: we can take advantage of orthogonal advances in optimization.

Our synthetic data takes the form $\ell_t(w) = |x_t\cdot w - x_t\cdot \w|$ where $x_t$ is randomly generated from a 100-dimensional $\mathcal{N}(0,\Sigma)$ and $\w$ is some pre-selected optimal point. We generate $\Sigma$ to be a random matrix with exponentially decaying eigenvalues and condition number 750. We consider two cases: either $\w$ is the eigenvector with smallest eigenvalue of $\Sigma$, or the eigenvector with largest eigenvalue. We compared \textsc{RecursiveOptimizer} to diagonal \textsc{Adagrad}, both of which come with good theoretical guarantees. For full implementation details, see Appendix \ref{sec:experimentsappendix}. The performance on a holdout set is shown in Figure \ref{fig:synthetic}. \textsc{RecursiveOptimizer} enjoys an advantage in the poorly-conditioned regime while maintaining performance in the well-conditioned problem. 

The dynamics of \textsc{RecursiveOptimizer} in the poorly-conditioned problem bear some discussion. Recall that our full-matrix regret bounds actually do not appear until the sum of the gradients grows to a certain degree. It appears that this may provide a period of ``slow convergence'' during which the inner optimizer is presumably finding the optimal $\v$, which is hard on poorly-conditioned problems. Once this is located, the algorithm makes progress very quickly.
\begin{figure*}[ht!]
\centering
  \subfigure[][]{\includegraphics[width=.4\textwidth]{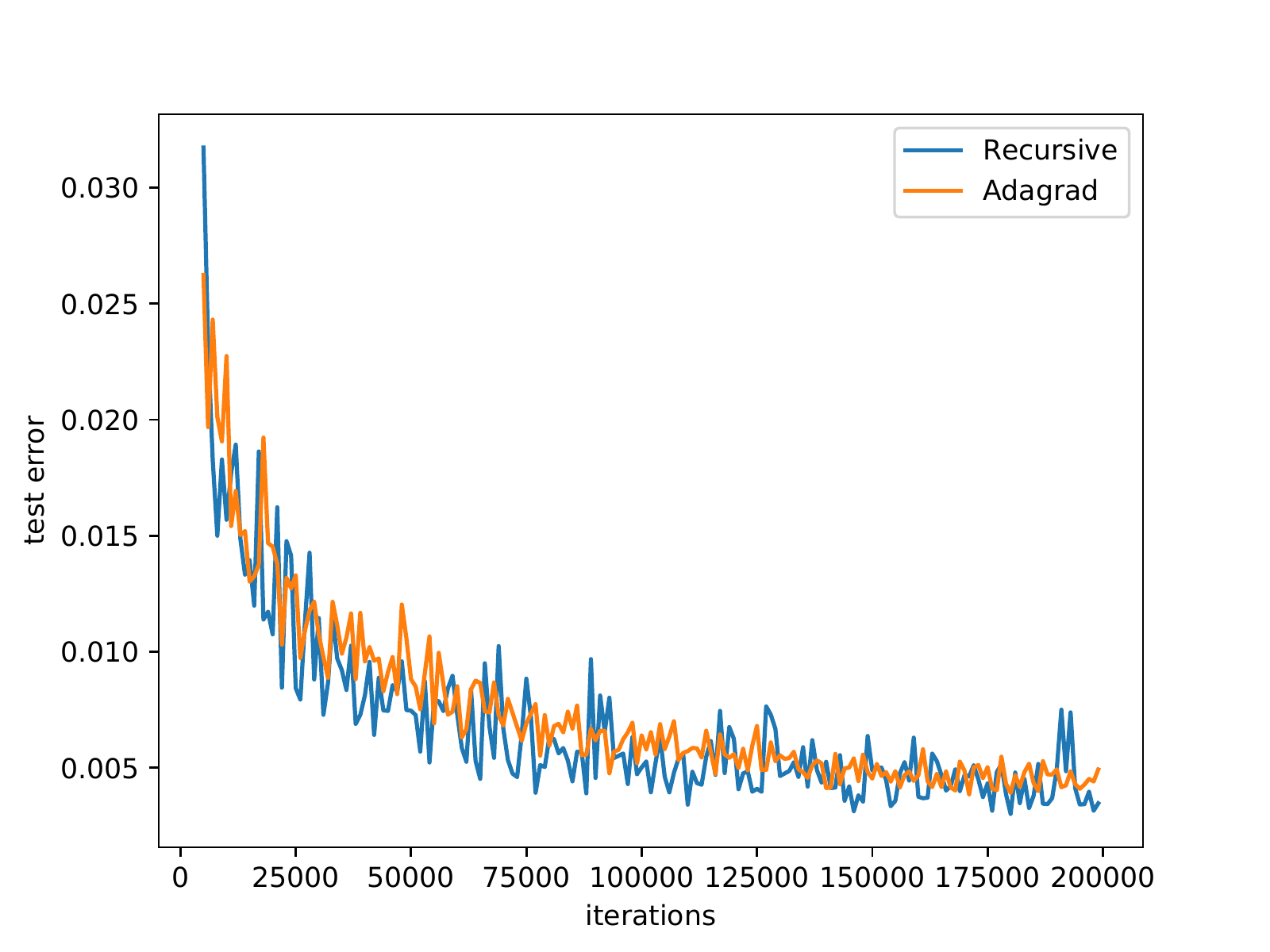}}
  \subfigure[][]{\includegraphics[width=.4\textwidth]{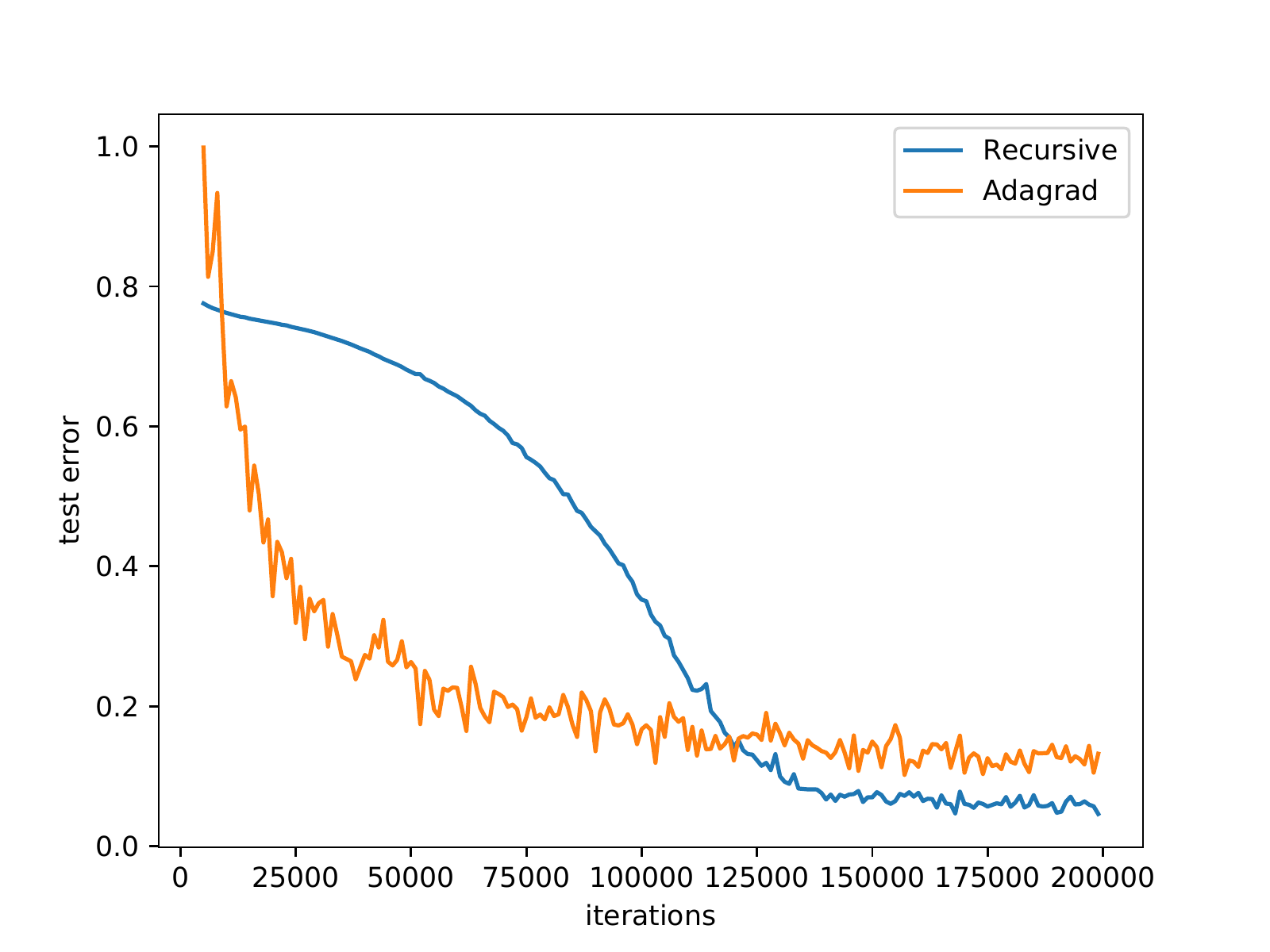}}
  \caption{Test Error in Synthetic Experiments. (a): $\w$ is eigenvector with maximum eigenvalue (well-conditioned). (b): $\w$ is eigenvector with minimum eigenvalue (poorly-conditioned).}
  \label{fig:synthetic}
\end{figure*}
\begin{figure*}[ht!]
  \centering
  \subfigure[][]{\includegraphics[width=.4\textwidth]{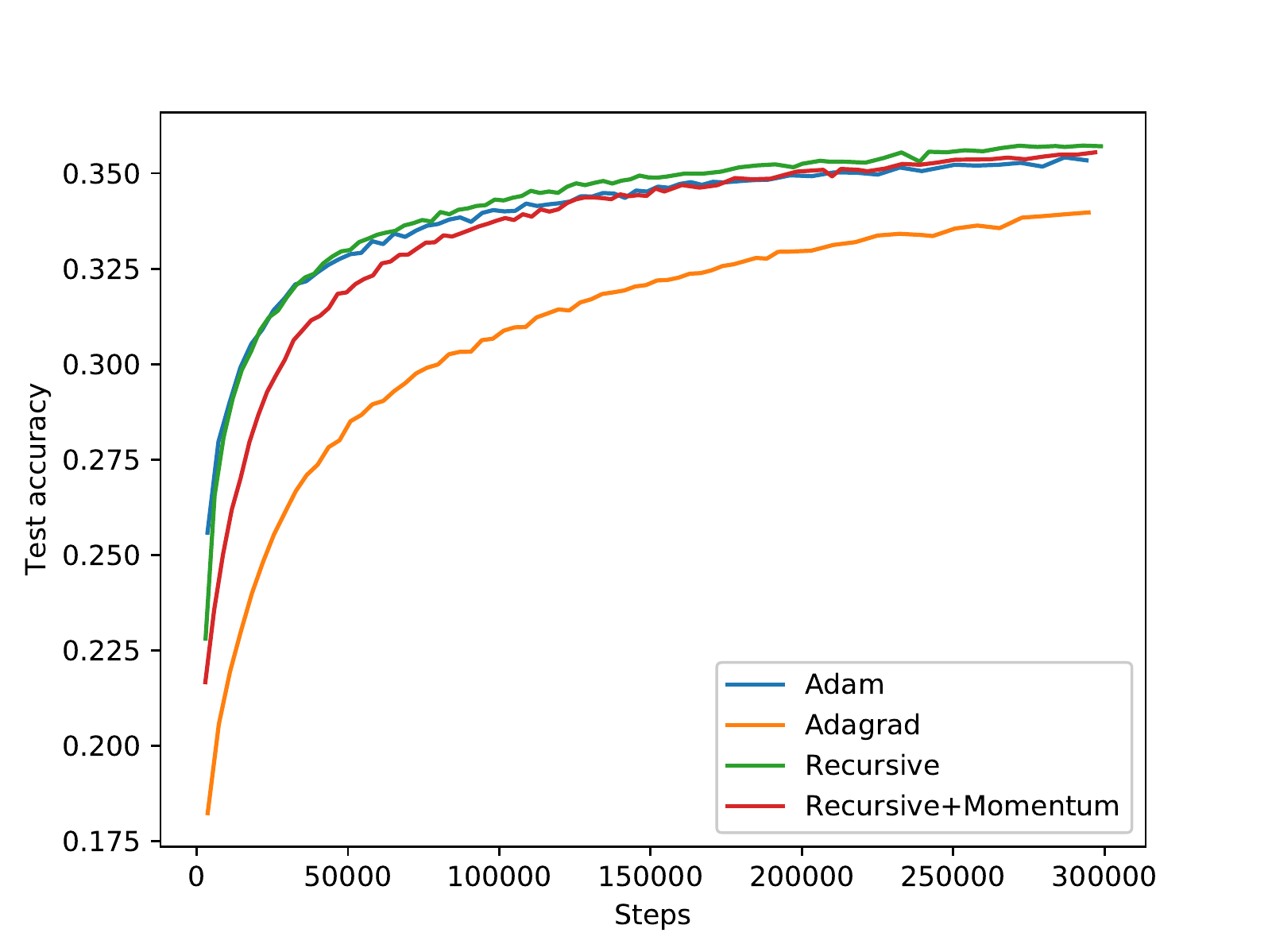}}
  \subfigure[][]{\includegraphics[width=.4\textwidth]{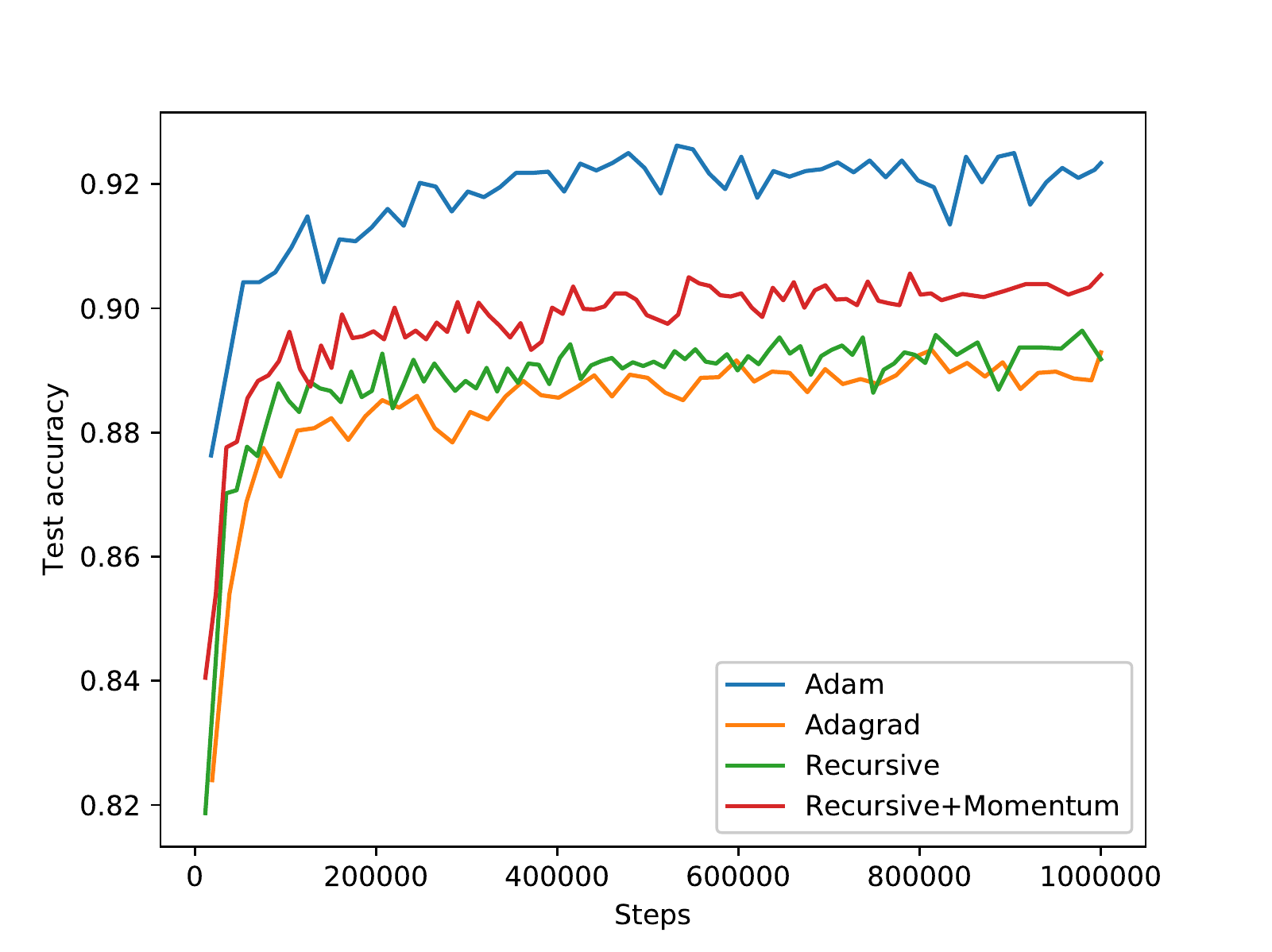}}
  \caption{Deep Learning Experiments (a): Transformer test accuracy on LM1B. (b). ResNet-32 test accuracy on CIFAR-10. See Appendix \ref{sec:momentum} for details on the momentum heuristic.}
  \label{fig:image_cifar10_languagemodel_lm1b32k-steps_text}
\end{figure*}
We also test \textsc{RecursiveOptimizer} on benchmark deep learning models. Specifically, we test performance with the ResNet-32~\citep{he2016deep} model on the CIFAR-10 image recognition dataset \citep{krizhevsky2009learning} and the Transformer model \cite{vaswani2017attention,tensor2tensor} on LM-1B~\citep{chelba2013one} and other textual datasets. We record train and test error, both as a function of number of iterations as well as a function of wall time. We compare performance to the commonly used Adam~\citep{kingma2014adam} and Adagrad optimizers (see Figure \ref{fig:image_cifar10_languagemodel_lm1b32k-steps_text}, and Appendix \ref{sec:experimentsappendix}). Even though our analysis relies heavily on duality and global properties of convexity, \textsc{RecursiveOptimizer} still performs well on these non-convex tasks: we are competitive in all benchmarks, and marginally the best in the Transformer tasks. This suggests an interesting line of future research: most popular optimizers used in deep learning operate in a proximal manner, producing each iterate as an offset from the previous iterate. This seems more appropriate to the non-convex setting as gradients provide only local information. It may therefore be valuable to develop a proximal version of \textsc{RecursiveOptimizer} that performs even better in the non-convex setting.

\section{Conclusion}
We have presented an algorithm that successfully obtains full-matrix style regret guarantees in certain settings without sacrificing runtime, space or regret in less favorable settings. The favorable settings we require for full-matrix performance are those in which the sum of the gradients exceeds the regret of some base algorithm, which should be $\tilde O(\sqrt{T})$. This suggests that any gradients with some systematic bias will satisfy our condition and exhibit full-matrix regret bounds for sufficiently large $T$. Our algorithmic design is based on techniques for unconstrained online learning, which necessitates an extra log factor in our regret over a mirror-descent algorithm tuned to the value of $\|\w\|$. As such, it is an interesting question whether a mirror-descent style analysis can achieve similar results with oracle tuning.

\section*{Acknowledgements}
We thank all the reviewers for their thoughtful comments and Ryan Sepassi for answering our questions about TensorFlow and Tensor2Tensor.

\bibliographystyle{icml2019}
\bibliography{references}

\begin{thebibliography}{31}
\providecommand{\natexlab}[1]{#1}
\providecommand{\url}[1]{\texttt{#1}}
\expandafter\ifx\csname urlstyle\endcsname\relax
  \providecommand{\doi}[1]{doi: #1}\else
  \providecommand{\doi}{doi: \begingroup \urlstyle{rm}\Url}\fi

\bibitem[Abadi et~al.(2016)Abadi, Barham, Chen, Chen, Davis, Dean, Devin,
  Ghemawat, Irving, Isard, et~al.]{abadi2016tensorflow}
Abadi, M., Barham, P., Chen, J., Chen, Z., Davis, A., Dean, J., Devin, M.,
  Ghemawat, S., Irving, G., Isard, M., et~al.
\newblock Tensorflow: a system for large-scale machine learning.
\newblock In \emph{OSDI}, 2016.

\bibitem[Agarwal et~al.(2018)Agarwal, Bullins, Chen, Hazan, Singh, Zhang, and
  Zhang]{agarwal2018case}
Agarwal, N., Bullins, B., Chen, X., Hazan, E., Singh, K., Zhang, C., and Zhang,
  Y.
\newblock The case for full-matrix adaptive regularization.
\newblock \emph{arXiv preprint arXiv:1806.02958}, 2018.

\bibitem[Cesa-Bianchi et~al.(2004)Cesa-Bianchi, Conconi, and
  Gentile]{cesa2004generalization}
Cesa-Bianchi, N., Conconi, A., and Gentile, C.
\newblock On the generalization ability of on-line learning algorithms.
\newblock \emph{IEEE Transactions on Information Theory}, 50\penalty0
  (9):\penalty0 2050--2057, 2004.

\bibitem[Chelba et~al.(2013)Chelba, Mikolov, Schuster, Ge, Brants, Koehn, and
  Robinson]{chelba2013one}
Chelba, C., Mikolov, T., Schuster, M., Ge, Q., Brants, T., Koehn, P., and
  Robinson, T.
\newblock One billion word benchmark for measuring progress in statistical
  language modeling.
\newblock \emph{arXiv preprint arXiv:1312.3005}, 2013.

\bibitem[Cutkosky \& Boahen(2017{\natexlab{a}})Cutkosky and
  Boahen]{cutkosky2017online}
Cutkosky, A. and Boahen, K.
\newblock Online learning without prior information.
\newblock In \emph{Conference on Learning Theory}, pp.\  643--677,
  2017{\natexlab{a}}.

\bibitem[Cutkosky \& Boahen(2017{\natexlab{b}})Cutkosky and
  Boahen]{cutkosky2017stochastic}
Cutkosky, A. and Boahen, K.~A.
\newblock Stochastic and adversarial online learning without hyperparameters.
\newblock In \emph{Advances in Neural Information Processing Systems}, pp.\
  5059--5067, 2017{\natexlab{b}}.

\bibitem[Cutkosky \& Orabona(2018)Cutkosky and Orabona]{cutkosky2018black}
Cutkosky, A. and Orabona, F.
\newblock Black-box reductions for parameter-free online learning in {B}anach
  spaces.
\newblock In \emph{COLT}, 2018.
\newblock URL \url{https://arxiv.org/abs/1802.06293}.

\bibitem[Duchi et~al.(2011)Duchi, Hazan, and Singer]{duchi2011adaptive}
Duchi, J., Hazan, E., and Singer, Y.
\newblock Adaptive subgradient methods for online learning and stochastic
  optimization.
\newblock \emph{Journal of Machine Learning Research}, 12\penalty0
  (Jul):\penalty0 2121--2159, 2011.

\bibitem[Gonen \& Shalev-Shwartz(2015)Gonen and
  Shalev-Shwartz]{gonen2015faster}
Gonen, A. and Shalev-Shwartz, S.
\newblock Faster sgd using sketched conditioning.
\newblock \emph{arXiv preprint arXiv:1506.02649}, 2015.

\bibitem[Gupta et~al.(2018)Gupta, Koren, and Singer]{gupta2018shampoo}
Gupta, V., Koren, T., and Singer, Y.
\newblock Shampoo: Preconditioned stochastic tensor optimization.
\newblock \emph{arXiv preprint arXiv:1802.09568}, 2018.

\bibitem[Hazan et~al.(2007)Hazan, Agarwal, and Kale]{hazan2007logarithmic}
Hazan, E., Agarwal, A., and Kale, S.
\newblock Logarithmic regret algorithms for online convex optimization.
\newblock \emph{Machine Learning}, 69\penalty0 (2-3):\penalty0 169--192, 2007.

\bibitem[Hazan et~al.(2016)]{hazan2016introduction}
Hazan, E. et~al.
\newblock Introduction to online convex optimization.
\newblock \emph{Foundations and Trends{\textregistered} in Optimization},
  2\penalty0 (3-4):\penalty0 157--325, 2016.

\bibitem[He et~al.(2016)He, Zhang, Ren, and Sun]{he2016deep}
He, K., Zhang, X., Ren, S., and Sun, J.
\newblock Deep residual learning for image recognition.
\newblock In \emph{Proceedings of the IEEE conference on computer vision and
  pattern recognition}, pp.\  770--778, 2016.

\bibitem[Kempka et~al.(2019)Kempka, Kot{\l}owski, and
  Warmuth]{kempka2019adaptive}
Kempka, M., Kot{\l}owski, W., and Warmuth, M.~K.
\newblock Adaptive scale-invariant online algorithms for learning linear
  models.
\newblock \emph{arXiv preprint arXiv:1902.07528}, 2019.

\bibitem[Kingma \& Ba(2014)Kingma and Ba]{kingma2014adam}
Kingma, D.~P. and Ba, J.~L.
\newblock Adam: A method for stochastic optimization.
\newblock In \emph{Proc. 3rd Int. Conf. Learn. Representations}, 2014.

\bibitem[Koren \& Livni(2017)Koren and Livni]{koren2017affine}
Koren, T. and Livni, R.
\newblock Affine-invariant online optimization and the low-rank experts
  problem.
\newblock In \emph{Advances in Neural Information Processing Systems}, pp.\
  4747--4755, 2017.

\bibitem[Krizhevsky \& Hinton(2009)Krizhevsky and
  Hinton]{krizhevsky2009learning}
Krizhevsky, A. and Hinton, G.
\newblock Learning multiple layers of features from tiny images.
\newblock Technical report, Citeseer, 2009.

\bibitem[Luo et~al.(2016)Luo, Agarwal, Cesa-Bianchi, and
  Langford]{luo2016efficient}
Luo, H., Agarwal, A., Cesa-Bianchi, N., and Langford, J.
\newblock Efficient second order online learning by sketching.
\newblock In \emph{Advances in Neural Information Processing Systems}, pp.\
  902--910, 2016.

\bibitem[Martens \& Grosse(2015)Martens and Grosse]{martens2015optimizing}
Martens, J. and Grosse, R.
\newblock Optimizing neural networks with kronecker-factored approximate
  curvature.
\newblock In \emph{International conference on machine learning}, pp.\
  2408--2417, 2015.

\bibitem[McMahan \& Streeter(2012)McMahan and Streeter]{mcmahan2012no}
McMahan, B. and Streeter, M.
\newblock No-regret algorithms for unconstrained online convex optimization.
\newblock In \emph{Advances in neural information processing systems}, pp.\
  2402--2410, 2012.

\bibitem[McMahan(2017)]{mcmahan2017survey}
McMahan, H.~B.
\newblock A survey of algorithms and analysis for adaptive online learning.
\newblock \emph{The Journal of Machine Learning Research}, 18\penalty0
  (1):\penalty0 3117--3166, 2017.

\bibitem[McMahan \& Orabona(2014)McMahan and Orabona]{mcmahan2014unconstrained}
McMahan, H.~B. and Orabona, F.
\newblock Unconstrained online linear learning in hilbert spaces: Minimax
  algorithms and normal approximations.
\newblock In \emph{Conference on Learning Theory}, pp.\  1020--1039, 2014.

\bibitem[McMahan \& Streeter(2010)McMahan and Streeter]{mcmahan2010adaptive}
McMahan, H.~B. and Streeter, M.
\newblock Adaptive bound optimization for online convex optimization.
\newblock \emph{arXiv preprint arXiv:1002.4908}, 2010.

\bibitem[Orabona(2013)]{orabona2013dimension}
Orabona, F.
\newblock Dimension-free exponentiated gradient.
\newblock In \emph{Advances in Neural Information Processing Systems}, pp.\
  1806--1814, 2013.

\bibitem[Orabona(2014)]{orabona2014simultaneous}
Orabona, F.
\newblock Simultaneous model selection and optimization through parameter-free
  stochastic learning.
\newblock In \emph{Advances in Neural Information Processing Systems}, pp.\
  1116--1124, 2014.

\bibitem[Orabona \& P{\'a}l(2016)Orabona and P{\'a}l]{orabona2016coin}
Orabona, F. and P{\'a}l, D.
\newblock Coin betting and parameter-free online learning.
\newblock In \emph{Advances in Neural Information Processing Systems}, pp.\
  577--585, 2016.

\bibitem[Orabona \& Tommasi(2017)Orabona and Tommasi]{orabona2017training}
Orabona, F. and Tommasi, T.
\newblock Training deep networks without learning rates through coin betting.
\newblock In \emph{Advances in Neural Information Processing Systems}, pp.\
  2160--2170, 2017.

\bibitem[Shalev-Shwartz et~al.(2012)]{shalev2012online}
Shalev-Shwartz, S. et~al.
\newblock Online learning and online convex optimization.
\newblock \emph{Foundations and Trends{\textregistered} in Machine Learning},
  4\penalty0 (2):\penalty0 107--194, 2012.

\bibitem[Vaswani et~al.(2017)Vaswani, Shazeer, Parmar, Uszkoreit, Jones, Gomez,
  Kaiser, and Polosukhin]{vaswani2017attention}
Vaswani, A., Shazeer, N., Parmar, N., Uszkoreit, J., Jones, L., Gomez, A.~N.,
  Kaiser, {\L}., and Polosukhin, I.
\newblock Attention is all you need.
\newblock In \emph{Advances in Neural Information Processing Systems}, pp.\
  5998--6008, 2017.

\bibitem[Vaswani et~al.(2018)Vaswani, Bengio, Brevdo, Chollet, Gomez, Gouws,
  Jones, Kaiser, Kalchbrenner, Parmar, Sepassi, Shazeer, and
  Uszkoreit]{tensor2tensor}
Vaswani, A., Bengio, S., Brevdo, E., Chollet, F., Gomez, A.~N., Gouws, S.,
  Jones, L., Kaiser, L., Kalchbrenner, N., Parmar, N., Sepassi, R., Shazeer,
  N., and Uszkoreit, J.
\newblock Tensor2tensor for neural machine translation.
\newblock \emph{CoRR}, abs/1803.07416, 2018.
\newblock URL \url{http://arxiv.org/abs/1803.07416}.

\bibitem[Zinkevich(2003)]{zinkevich2003online}
Zinkevich, M.
\newblock Online convex programming and generalized infinitesimal gradient
  ascent.
\newblock In \emph{Proceedings of the 20th International Conference on Machine
  Learning (ICML-03)}, pp.\  928--936, 2003.

\end{thebibliography}

\vfill\eject
\appendix
\onecolumn

This appendix is organized as follows:
\begin{enumerate}
    \item In Section \ref{sec:technical} we prove some technical Lemmas used in our main results.
    \item In Section \ref{sec:proofrecursive} we prove Theorem \ref{thm:recursiveregret}.
    \item In Section \ref{sec:proofdiag} we prove Theorem \ref{thm:diaganalysis}, and also provide a doubling-trick based algorithm that achieves the optimal log factors in its regret bound.
    \item In Section \ref{sec:experimentsappendix} we provide details about our empirical evaluation.
\end{enumerate}

\section{Technical Lemmas}\label{sec:technical}
We compute a useful Fenchel conjugate below:
\begin{lemma}\label{thm:conjugate}
Let $f(x) = a\exp(bx)$ for $a\ge 0$ and $b\ge 0$. Then $f^\star(y) = \frac{y}{b}\left(\log\left(\frac{y}{ab}\right)-1\right)$ for all $y\ge 0$.
\end{lemma}
\begin{proof}
We want to maximize
\[
yx - a\exp(bx)
\]
 as a function of $x$. Differentiating, we have $y-ab\exp(bx)=0$, so that $x=\frac{1}{b}\log\left(\frac{y}{ab}\right)$ (where we've used our assumption about non-negativity of all variables). Then we simply substitute this value in to conclude the Lemma.
\end{proof}

Next, we have a useful optimization solution:

\begin{lemma}\label{thm:balancelog}
Suppose $A,B,C,D$ are non-negative constants. Then
\begin{align*}
\inf_{x\in[0,1/2]} \frac{A}{x}\left[\log\left(\frac{B}{x}\right)-C\right] + Dx &\le2\max\left[ \sqrt{AD\max\left[\log\left(\frac{B\sqrt{D}}{\sqrt{A}}\right)-C,1\right]}, \right.\\
&\quad\quad\quad\left.2A\max\left[\log\left(\frac{B\sqrt{4A^2+D}}{\sqrt{A}}\right)-C,1\right]\right]
\end{align*}
\end{lemma}
\begin{proof}
We will just guess a value for $x$:
\[
x = \frac{\sqrt{A}}{\sqrt{D}}\sqrt{\max\left[\log\left(\frac{B\sqrt{D}}{\sqrt{A}}\right)-C,1\right]}
\]
Suppose that this quantity is in $[0,1/2]$ for now.
Then we have
\[
\log\left(\frac{B}{x}\right)\le \log\left(\frac{B\sqrt{D}}{\sqrt{A}}\right)
\]
so that
\begin{align*}
\frac{A}{x}\left[\log\left(\frac{B}{x}\right)-C\right]&\le \frac{A}{x}\left[\log\left(\frac{B\sqrt{D}}{\sqrt{A}}\right)-C\right]\\
&\le \sqrt{AD\left(\log\left(\frac{B\sqrt{D}}{\sqrt{A}}\right)-C\right)}
\end{align*}
Thus we have:
\[
\frac{A}{x}\left[\log\left(\frac{B}{x}\right)-C\right] + Dx\le 2 \sqrt{AD\max\left[\log\left(\frac{B\sqrt{D}}{\sqrt{A}}\right)-C,1\right]}
\]

Now suppose instead that our guess is outside $[0,1/2]$. Then we must have
\begin{align*}
4A\max\left[\log\left(\frac{B\sqrt{D}}{\sqrt{A}}\right)-C,1\right]\ge D
\end{align*}
and also
\begin{align*}
    2\sqrt{\max\left[\log\left(\frac{B\sqrt{D}}{\sqrt{A}}\right)-C,1\right]}&\ge \frac{\sqrt{D}}{\sqrt{A}}
\end{align*}
So now with $x=1/2$ we obtain:
\begin{align*}
    \frac{A}{x}\left[\log\left(\frac{B}{x}\right)-C\right] + Dx&\le 2A(\log(2B)-C)+2A\max\left[\log\left(\frac{B\sqrt{D}}{\sqrt{A}}\right)-C,1\right]\\
    &\le 2A(\log(2B)-C)+2A\max\left[\log\left(\frac{B\sqrt{D}}{\sqrt{A}}\right)-C,1\right]\\
    &\le 4A\max\left[\log\left(\frac{B\sqrt{4A^2+D}}{\sqrt{A}}\right)-C,1\right]
\end{align*}

\end{proof}

\subsection{Proof of Lemma \ref{thm:generalrecursiveregret}}
\begin{proof}
Recall that $\wealth_T(v)$ is the wealth of an algorithm that always uses betting fraction $v$. So long as $\|v\|\le 1/2$, we have
\[
\wealth_T(v) \ge \epsilon\exp\left(-v\cdot \sum_{t=1}^T g_t - \sum_{t=1}^T (v\cdot g_t)^2\right)
\]
Setting $X= -\sum_{t=1}^T g_t\cdot \frac{\w}{\|\w\|}$, and $Z=\sum_{t=1}^T (g_t\cdot \w/\|\w\|)^2$ yields:
\[
\wealth_T\left(c \frac{\w}{\|\w\|}\right)\ge \epsilon\exp\left(c X - c^2 Z\right)
\]
By mild abuse of notation, we define the regret of our $v$-choosing algorithm at $c \frac{\w}{\|\w\|}$ as $R^v_T(c)$, so that following (\ref{eqn:regretrecursion}) we can write:
\begin{equation}
\wealth_T \ge \epsilon\exp\left(c X - c^2 Z - R^v_T(c)\right) = f_c(X)\label{eqn:wealthbound}
\end{equation}
where we defined $f_c(X)=\epsilon\exp\left(c X - c^2 Z - R^v_T(c)\right)$. Now by Lemmas \ref{thm:duality} and \ref{thm:conjugate}, we obtain:
\begin{align}
R_T(\w) &\le \epsilon+ f_c^\star(\|\w\|)\nonumber\\
&= \epsilon + \frac{\|\w\|}{c}\left[\log\left(\frac{\|\w\|}{\epsilon c \exp(-c^2 Z - R^v_T(c))}\right)-1\right]\nonumber\\
&= \epsilon + \frac{\|\w\|}{c}\left[\log\left(\frac{\|\w\|}{c\epsilon}\right) + c^2 Z  + R^v_T(c)-1\right]\nonumber\\
&= \epsilon + \frac{\|\w\|}{c}(\log(\|\w\|/c\epsilon) -1) + \|\w\|c Z + \frac{\|\w\|}{c}R^v_T(c)\label{eqn:expandedbound}
\end{align}
\end{proof}

\section{Proof of Theorem \ref{thm:recursiveregret}}\label{sec:proofrecursive}
The following theorem provides a more detailed version of Theorem \ref{thm:recursiveregret}, including all constants:
\begin{theorem}\label{thm:tighterrecursiveregret}
Suppose $\|g_t\|_\star \le 1$ for some norm $\|\cdot\|$ for all $t$. Further suppose that \textsc{InnerOptimizer} has outputs satisfying $\|v_t\|\le 1/2$ and guarantees regret nearly linear in $\|\v\|$:
\begin{align*}
R^v_T(\v)=\sum_{t=1}^T z_t\cdot v_t-z_t\cdot \v  &\le \epsilon + \|\v\| G_T(\v/\|\v\|)
\end{align*}
for some function $G_T(\v/\|\v\|)$ for any $\v$ with $\|\v\|\le 1/2$.
Then if $-\sum_{t=1}^T g_t\cdot \frac{\w}{\|\w\|}\ge 2G_T(\w/\|\w\|)$, \textsc{RecursiveOptimizer} obtains
\begin{align*}
    R_T(\w) &\le \epsilon + 4\sqrt{\left(4\|\w\|^2+\sum_{t=1}^T(g_t\cdot\w)^2\right)\max\left[\log\left(\frac{2\sqrt{4\|\w\|^2+\sum_{t=1}^T (g_t\cdot \w)^2}}{\epsilon}\right)+\epsilon-1,1\right]}
\end{align*}
and otherwise
\[
R_T(\w) \le \epsilon + 2\|\w\|G_T(\w/\|\w\|)
\]
\end{theorem}
\begin{proof}
First, observe that since $\|v_t\|\le 1/2$ for all $t$, we must have $\wealth_T\ge 0$ for all $t$ and so
\[
\sum_{t=1}^T g_t\cdot w_t\le \epsilon
\]
Therefore, if $-\sum_{t=1}^T g_t\cdot \frac{\w}{\|\w\|}<2G_T(\w/\|\w\|)$ we must have
\begin{align*}
    R_T(\w) &= \sum_{t=1}^T g_t\cdot w_t - g_t\cdot \w\\
    &=\sum_{t=1}^T g_t\cdot w_t - \|\w\|\sum_{t=1}^T g_t\cdot \frac{\w}{\|\w\|}\\
    &\le \epsilon + 2\|\w\|G(\w/\|\w\|)
\end{align*}

Which proves one case of the Theorem. So now we assume $-\sum_{t=1}^T g_t\cdot \frac{\w}{\|\w\|}\ge 2G_T(\w/\|\w\|)$.

Recall the inequality:
\begin{align*}
\log(\wealth_T)\ge \log(\epsilon)+ \sum_{t=1}^T -g_t\cdot \v -(g_t\cdot \v)^2 - R^v_T(\v)
\end{align*}
for any $\v$ with $\|\v\|\le 1/2$.
Using our assumption on $R^v_T$, and setting $\v=c\w/\|\w\|$ for some unspecified $c\in[0,1/2]$, we have
\begin{align*}
    \log(\wealth_T) &\ge \log(\epsilon)+ \sum_{t=1}^T -g_t\cdot \v -(g_t\cdot \v)^2 -(\epsilon +\|\v\|G(\v/\|\v\|))\\
    &\ge -\epsilon + \log(\epsilon) + \sum_{t=1}^T -c g_t\cdot \frac{\w}{\|\w\|} - c^2Z - c G\left(\frac{\w}{\|\w\|}\right)\\
    &\ge -\epsilon + \log(\epsilon) + \sum_{t=1}^T -\frac{c}{2}g_t\cdot \frac{\w}{\|\w\|} - c^2Z
\end{align*}
where we have defined $Z=\sum_{t=1}^T \left(g_t\cdot \frac{\w}{\|\w\|}\right)^2$. Now we define
\begin{align*}
    f(X) = \epsilon \exp\left(-\epsilon -c^2Z + \frac{c}{2}X\right)
\end{align*}
to obtain
\begin{align*}
    \wealth_T\ge f\left(-\sum_{t=1}^T g_t\cdot\frac{\w}{\|\w\|}\right)
\end{align*}
Then using Lemmas \ref{thm:duality} and \ref{thm:conjugate} we obtain:
\begin{align*}
    R_T(\w) & \le \epsilon+f^\star(\|\w\|)\\
    &\le \epsilon + \frac{2\|\w\|}{c}\left[\log\left(\frac{2\|\w\|}{c\epsilon}\right)+\epsilon-1\right]+2\|\w\|cZ 
\end{align*}
Now we optimize $c\in[0,1/2]$ using Lemma \ref{thm:balancelog}:

\begin{align*}
    R_T(\w) &\le \epsilon + 4\|\w\|\sqrt{(4+Z)\max\left[\log\left(\frac{2\|\w\|\sqrt{4+Z}}{\epsilon}\right)+\epsilon-1,1\right]}
\end{align*}
\end{proof}

\section{Proof of Theorem \ref{thm:diaganalysis}}\label{sec:proofdiag}
The following theorem provides a more detailed version of Theorem \ref{thm:diaganalysis}, including all constants and logarithmic factors.
\begin{theorem}\label{thm:tighterdiag}
Suppose $\|g_t\|_\infty\le 1$ for all $t$. Then for all $\|\w\|_\infty\le 1/2$, Algorithm \ref{alg:diag} guarantees regret:
\begin{align*}
R_T(\w)&\le d\epsilon+2\sum_{i=1}^d |\w_i|\max\left[\sqrt{\left[\frac{5}{4\eta}+G_i\left(1+\frac{2}{\eta}\right)\right]\max\left[\log\left(\frac{|\w_i|(1+4G_i)^\eta \sqrt{2/\eta+G_i(1+2/\eta)}}{\epsilon}\right)-1,1\right]},\right.\\
    &\quad\quad\quad\quad\quad\left.2\max\left[\log\left(\frac{|\w_i|(1+4G_i)^\eta \sqrt{4+5/4\eta+G_i(1+2/\eta)}}{\epsilon}\right)-1,1\right]\right]\\
    &\le  d\epsilon+2\|\w\|_\infty\sum_{i=1}^d \frac{|\w_i|}{\|\w\|_\infty}\max\left[\sqrt{\left[\frac{5}{4\eta}+G_i\left(1+\frac{2}{\eta}\right)\right]\max\left[\log\left(\frac{(1+4G_i)^\eta \sqrt{5/4\eta+G_i(1+2/\eta)}}{2\epsilon}\right)-1,1\right]},\right.\\
    &\quad\quad\quad\quad\quad\left.2\max\left[\log\left(\frac{(1+4G_i)^\eta \sqrt{4+5/4\eta+G_i(1+2/\eta)}}{2\epsilon}\right)-1,1\right]\right]\\
    &:=\epsilon d + \|\w\|_\infty G(\w/\|\w\|_\infty)
\end{align*}
\end{theorem}
\begin{proof}

First, observe that Algorithm \ref{alg:diag} is running $d$ copies of a 1-dimensional algorithm, one per coordinate. Using the classic diagonal trick, we can write
\begin{align*}
    R_T(\w) \le \sum_{t=1}^T \langle g_t, w_t -\w\rangle =\sum_{i=1}^d \sum_{t=1}^T g_{t,i}(w_{t,i}- \w)=\sum_{i=1}^d R_{T,i}(\w)
\end{align*}
where $R_{T,i}$ indicates the regret of the $i$th 1-dimensional optimizer. As a result, we will only analyze each dimension individually and combine all the dimensions at the end. To make notation cleaner during this process, we drop the subscripts $i$.

Next, we claim that it suffices to examine the regret of the $x_t$s rather than that of the $w_t$s. In particular, it holds that:
\[
g_t(w_t - \w)\le \tilde g_t(x_t -\w)
\]
We show this via case-work. First, if $w_t=x_t$ the claim is immediate because $g_t=\tilde g_t$. Suppose $g_t(x_t-w_t)\ge 0$. Then $g_t=\tilde g_t$ and $g_tx_t\ge g_t w_t$ so that the claim follows. Finally, suppose $g_t(x_t-w_t)<0$. Then since $x_t\ne w_t$, we must have $w_t=\clip(x_t, -1/2,1/2)$ so that $\sign(x_t)=\sign(x_t-w_t)=\sign(w_t)$ and so $\sign(g_t)=-\sign(w_t)$. Further, since $w_t\in\{-1/2,1/2\}$ and $\w\in[-1/2,1/2]$, $\sign(w_t-\w)=\sign(w_t)$. Therefore $g_t(w_t- \w)\le 0 = \tilde g_t(x_t-\w)$. Therefore we can write:
\begin{align*}
    \sum_{t=1}^T g_t(w_t-\w)\le \sum_{t=1}^T \tilde g_t(x_t-\w)
\end{align*}
The RHS of the above is the regret of the $x_t$s with respect to the $\tilde g_t$s, so we reduce to analyzing this regret. Eventually the regret bound will be increasing in $|\tilde g_t|$, and since $|\tilde g_t|\le |g_t|$, we can seamlessly transition to a regret bound in terms of the $g_t$.

Finally, observe that the $x_t$s are generated by a betting algorithm using betting-fractions $v_t$. Inspection of the formula for $v_t$ reveals that we can write:
\[
v_t = \argmin_{v\in[-1/2,1/2]} \frac{1}{4\eta }A_tv^2+\sum_{t=1}^T z_tv 
\]
so that the $v_t$ are actually the outputs of an FTRL algorithm using regularizers $\frac{A_t}{4\eta }v^2$, which are $\frac{A_t}{2\eta}$-strongly convex. That is, the $x_t$s are actually an instance of \textsc{RecursiveOptimizer}.

Thus by Lemma \ref{thm:generalrecursiveregret} we have
\begin{align*}
    R_T(\w) & \le \inf_{c\in[0,1/2]} \epsilon + \frac{|\w|}{c}\left(\log\left(\frac{|\w|}{c\epsilon}\right)-1\right) + |\w|cZ + \frac{|\w|}{c}R_T^v\left(c\frac{\w}{|\w|}\right)
\end{align*}
where $Z=\sum_{t=1}^T g_t^2$ in this one-dimensional case.

Next we tackle $R_T^v$. To do this, we invoke the FTRL analysis of \citep{mcmahan2017survey} to claim:
\begin{align*}
R_T^v(x) &\le \frac{A_T}{4\eta }x^2 + \sum_{t=1}^T \frac{z_t^2\eta }{A_{t-1}^2}\\
&\le \frac{A_T}{4\eta}x^2 + \eta \sum_{t=1}^T \frac{z_t^2}{5+\sum_{i=1}^{t-1} z_i^2}
\end{align*}
Now observe that each $z_t$ satisfies $|z_t|\le 2|g_t|\le 2$ so that 
\begin{align*}
    \sum_{t=1}^T \frac{z_t^2}{5+\sum_{i=1}^{t-1} z_i^2}&\le \sum_{t=1}^T \frac{z_t^2}{1+\sum_{i=1}^{t} z_i^2}\\
    &\le \log\left(1+\sum_{t=1}^T z_t^2\right)\\
    &\le \log\left(1+4\sum_{t=1}^T g_t^2\right)
\end{align*}
Therefore we have
\begin{align*}
    R_T^v(x)&\le \frac{5+4\sum_{t=1}^T g_t^2}{4\eta} x^2 + \eta  \log\left(1+4\sum_{t=1}^T g_t^2\right)\\
    &= \frac{5+4Z}{4\eta}x^2+ \log\left((1+4Z)^\eta\right)
\end{align*}

Plugging back into the result from Lemma \ref{thm:generalrecursiveregret} we obtain:
\begin{align*}
    R_T(\w) & \le \inf_{c\in[0,1/2]} \epsilon + \frac{|\w|}{c}\left(\log\left(\frac{|\w|(1+4Z)^\eta}{c\epsilon}\right)-1\right) + |\w|c(5/4\eta+Z(1+2/\eta))
\end{align*}

Then using Lemma \ref{thm:balancelog} we get:
\begin{align*}
    R_T(\w) &\le \epsilon + 2|\w|\max\left[\sqrt{\left[\frac{5}{4\eta}+Z\left(1+\frac{2}{\eta}\right)\right]\max\left[\log\left(\frac{|\w|(1+4Z)^\eta \sqrt{2/\eta+Z(1+2/\eta)}}{\epsilon}\right)-1,1\right]},\right.\\
    &\quad\quad\quad\left.2\max\left[\log\left(\frac{|\w|(1+4Z)^\eta \sqrt{4+5/4\eta+Z(1+2/\eta)}}{\epsilon}\right)-1,1\right]\right]
\end{align*}

Now we simply combine each of the $d$ dimensional regret bounds and observe that in a one-dimension, $Z= \sum_{t=1}^T g_t{,i}^2 = G_i$ to obtain:
\begin{align*}
    R_T(\w) \le d\epsilon+2\sum_{i=1}^d& |\w_i|\max\left[\sqrt{\left[\frac{5}{4\eta}+G_i\left(1+\frac{2}{\eta}\right)\right]\max\left[\log\left(\frac{|\w_i|(1+4G_i)^\eta \sqrt{5/4\eta+G_i(1+2/\eta)}}{\epsilon}\right)-1,1\right]},\right.\\
    &\quad\left.2\max\left[\log\left(\frac{|\w_i|(1+4G_i)^\eta \sqrt{4+5/4\eta+G_i(1+2/\eta)}}{\epsilon}\right)-1,1\right]\right]
\end{align*}
\end{proof}

\subsection{Optimal Logarithmic Factors}\label{sec:optlog}

The previous analysis obtains logarithmic factors of the form $\log(|w|Z^{1/2+\eta}/\epsilon)$ for any given $\eta>0$. For $|w|>\epsilon$, this is the same up to constant factors as the optimal bound $\log(|w|\sqrt{Z}/\epsilon)$. However, for small $w$ this is not so. In the small-$\w$ case, our bound is already an improvement on the previous exponent \citep{cutkosky2018black}, which has an exponent of $4.5$ instead of $1/2+\eta$, but here we sketch how to remove $\eta$ completely using the classic doubling trick. We present the idea in one dimensional unconstrained problems only: conversion to constrained or high dimensional problems may be accomplished via per-coordinate updates as in Theorem \ref{thm:tighterdiag}, or via the dimension-free reduction in \citep{cutkosky2018black}. The idea is essentially the same as Algorithm \ref{alg:diag}, but instead of using a varying $A_t$, we use a \emph{fixed} $A$ and set $\eta=1$. We restart the algorithm with a doubled value for $A$ whenever we observe $2Z=2\sum g_t^2 >A$. Let us analyze this scheme during one epoch of fixed $A$-value. Following identical analysis as in Theorem \ref{thm:tighterdiag}, we observe that
\begin{align*}
    R_T^v(x) &\le \frac{A}{2}x^2+\frac{1}{2}\sum_{t=1}^T \frac{z_t^2}{A}\\
    &\le \frac{A}{2}x^2 + \frac{1}{2}\sum_{t=1}^T \frac{4g_t^2}{A}\\
    &\le \sum_{t=1}^T g_t^2x^2 + 1=Zx^2+1
\end{align*}
Then applying Theorem \ref{thm:generalrecursiveregret} we have
\begin{align*}
    R^k_T(\w) & \le \inf_{c\in[0,1/2]} \epsilon + \frac{|\w|}{c}\left(\log\left(\frac{|\w|}{c\epsilon}\right)-1\right) + 2|\w|cZ_k + \frac{|\w|}{c}\\
    &= \epsilon + \frac{|\w|}{c}\log\left(\frac{|\w|}{c\epsilon}\right) + 2|\w|cZ_k
\end{align*}
where $R^k_T$ indicates regret in the $k$th epoch and $Z_k$ is the value of $Z$ in the $k$th epoch.
Optimizing $c$, we obtain:
\begin{align*}
    R_T(\w) & \le\epsilon + 2|\w|\max\left[\sqrt{2Z_k\max\left[\log\left(\frac{|\w|\sqrt{2Z_k}}{\epsilon}\right),1\right]},\ 2\max\left[\log\left(\frac{|\w|\sqrt{4+2Z_k}}{\epsilon}\right),1\right] \right]
\end{align*}
Let $Z$ be the true value of $Z$ (i.e. $Z=\sum_{t=1}^T g_t^2$ across all epochs, in contrast to a $Z_k$). Then we have
\begin{align*}
    R^k_T(\w) & \le\epsilon + 2|\w|\max\left[\sqrt{2Z_k\max\left[\log\left(\frac{|\w|\sqrt{2Z}}{\epsilon}\right),1\right]},\ 2\max\left[\log\left(\frac{|\w|\sqrt{4+2Z}}{\epsilon}\right),1\right] \right]
\end{align*}
Then summing over all epochs, we obtain
\begin{align*}
    R_T(\w) &\le O\left(\epsilon \log(Z) + |\w|\max\left[\sqrt{Z\max\left[\log\left(\frac{|\w|\sqrt{2Z}}{\epsilon}\right),1\right]},\ \log(Z)\max\left[\log\left(\frac{|\w|\sqrt{4+2Z}}{\epsilon}\right),1\right] \right]\right)
\end{align*}

\section{Experimental Details}\label{sec:experimentsappendix}

In this Section we describe our experiments in detail. All of our neural network experiments were conducted using the Tensor2Tensor library \citep{tensor2tensor}. We evaluated \textsc{RecursiveOptimizer} on several datasets included in the library, including MNIST and CIFAR-10 image classification, LM1B language modeling with 32k, and IMDB sentiment analysis tasks. On CIFAR-10, we used a ResNet model \citep{he2016deep} (ResNet-32), on MNIST we used a simple two layer fully connected network as well as logistic regression, and for the remaining tasks we used the Transformer model \citep{vaswani2017attention}.

We used $\epsilon=1.0$ in \textsc{RecursiveOptimizer}. For our baseline optimizers Adam and Adagrad, we used default parameters provided by Tensor2Tensor for each dataset when available. Often these were not available for Adagrad, in which case we manually tuned the learning rate on a small exponentially spaced grid. Experiments with larger models or data sets, i.e. CIFAR-10 and LM1B, ran on single NVIDIA P100 GPU, the rest on single NVIDIA K1200 GPU.


\subsection{Choice of Inner Optimizer}
Our analysis uses a Follow-the-Regularized-Leader algorithm in the inner optimizer \textsc{DiagOptimizer} to choose the inner-most betting fraction $v_t$. However, according to Theorem \ref{thm:recursiveregret}, we may use \emph{any} optimizer with a sufficiently good regret guarantee as the inner optimizer. Since our initial submission, \cite{kempka2019adaptive} proposed the \textsc{ScInOL} algorithm that obtains regret similar to \textsc{DiagOptimizer} (albeit with somewhat worse logarithmic factors). However, we found that using \textsc{ScInOL} resulted in much better performance on the Transformer model tasks, so we used it as the inner optimizer in all experiments. We conjecture that our algorithm is inheriting some of the scale-invariance properties of \textsc{ScInOL}, which allows it to be more robust. We stress that this is still theoretically sound - the only change will be a small increase in the logarithmic factors.

\subsection{Momentum Analog}\label{sec:momentum}
In our experiments we found that augmenting \textsc{RecursiveOptimizer} with the ``momentum''-like offsets for parameter-free online learning proposed by \citep{cutkosky2017stochastic,cutkosky2018black} improved the empirical performance on CIFAR-10, so all of our results show two curves for \textsc{RecursiveOptimizer}, both with and without momentum (except for the synthetic experiments, in which we did not use momentum). In brief, this consists of replacing each iterate $w_t$ with $w_t+\overline{w}_t$ where
\begin{align*}
    \overline{w}_t = \sum_{t'=1}^t \frac{\|g_{t'}\|_\star^2 w_{t'}}{\sum_{t'=1}^t \|g_{t'}\|_\star^2}
\end{align*}



\subsection{Dealing with unknown bound on $g_t$}
Our theory requires $\|g_t\|_\star\le 1$ where $\|\cdot\|$ is the $\infty$ norm. 
Although we may replace $1$ with any known bound $g_{\text{max}}$, it is not possible to simply ignore this requirement in implementing the algorithm: doing so may cause wealth to become negative, which will completely destabilize the algorithm since it will be implicitly differentiating the logarithm of a negative number. 
However, we do not wish to have to provide this bound to the algorithm, so we adopt a simple heuristic. 
We maintain $g_{\max}$, the maximum value of $\|g_t\|_1$ we have observed so far during the course of the optimization. 
Then instead of providing $g_t$ to \textsc{RecursiveOptimizer}, we provide $g_t/g_{\max}$. 
Ideally, $g_{\max}$ will only increase during the very beginning of the optimization, after which we will simply be rescaling the gradients by a constant factor. 
Since our regret bounds are nearly scale-free, this should hopefully have negligible effect on the performance. Note that it is actually impossible to design an algorithm that maintains regret nearly linear in $\|\w\|$ while also being adaptive to the unknown final value of $g_{\max}$ \citep{cutkosky2017online}.

\subsection{Initial Betting Fraction}
Prior results on coin-betting in deep learning \citep{orabona2017training} suggest that a valuable heuristic is to keep the initial betting fraction smaller than some moderate constant. This has the effect of preventing the initial step taken by the algorithm from being too large. We choose to apply this heuristic to the betting fraction of the inner-optimizer - \emph{not} the betting fraction of the outer optimizer. Note that \textsc{ScInOL} is also a coin-betting algorithm, so it still makes sense to apply the heuristic in this manner. We clip the inner betting fraction of dimension $i$ to be always at most $\eta=0.1$ until $\sum v_{t,i}^2\ge 1$ where $v_t$ is the gradient passed to the inner betting fraction. This trick has no theoretical basis, but seems to provide significant improvement in the deep learning experiments.

\subsection{Empirical Results}
Now we plot our performance on the benchmarks. We record performance on train and test set, both in terms of number of iterations as well as wall clock time. Generally from eight possible combinations of {train, test}-{top-1 accuracy, loss}-{steps, time} curves we show train loss and test accuracy both by steps and time, other combinations are indistinguishably similar and omitted for brevity. For LM1B and Penn Tree Bank language models we include the log perplexity metric as well.


With regard to efficiency observe that the right hand side plots of Figures \ref{fig:image_cifar10} and \ref{fig:languagemodel_lm1b32k} whose x-axis is wall clock time are rather similar to left hand side plots based on number of iterations. For a more accurate view, Figure \ref{fig:speed} shows that \textsc{RecursiveOptimizer} is somewhat slower than both Adam and Adagrad. It is evident that the algorithm requires more computation, although only by a constant factor.
We made essentially no effort to optimize our code.
We expect that with more careful implementation these numbers can be improved. Secondly, Adam (more specifically LazyAdam used by Tensor2Tensor framework) and Adagrad optimizers handle sparse and dense gradients differently. Our current implementation treats sparse gradients as if they were dense ignoring their sparsity which is detrimental for large vocabulary embeddings. 

Observe that on the convex logistic regression task, all optimizers converge to the same minimum of train loss, as theory predicts. On the non-convex neural network tasks, \textsc{RecursiveOptimizer} seems to be marginally better than the baselines on the Transformer task, but slightly worse than Adam on CIFAR-10. Interestingly, the momentum heuristic was helpful on CIFAR-10, but seemed detrimental on the Transformer tasks. We suspect that \textsc{RecursiveOptimizer} is held back on these non-convex tasks by the somewhat global nature of our update. Because our iterates are $v_t\wealth_{t-1}$, it is easily feasible for the iterate to change quite dramatically in a single round as wealth becomes larger. In contrast, proximal methods such as Adam or Adagrad enforce some natural stability in their iterates. In future, we plan to develop a version of our techniques that also enforces some natural stability, which may be more able to realize gains in the non-convex setting.


\begin{figure*}
  \centering
  \subfigure[][Test accuracy vs steps]{\includegraphics[width=.4\textwidth]{figs/image_cifar10/Test-accuracy-Steps-eps-converted-to.pdf}}\quad
  \subfigure[][Test accuracy vs time]{\includegraphics[width=.4\textwidth]{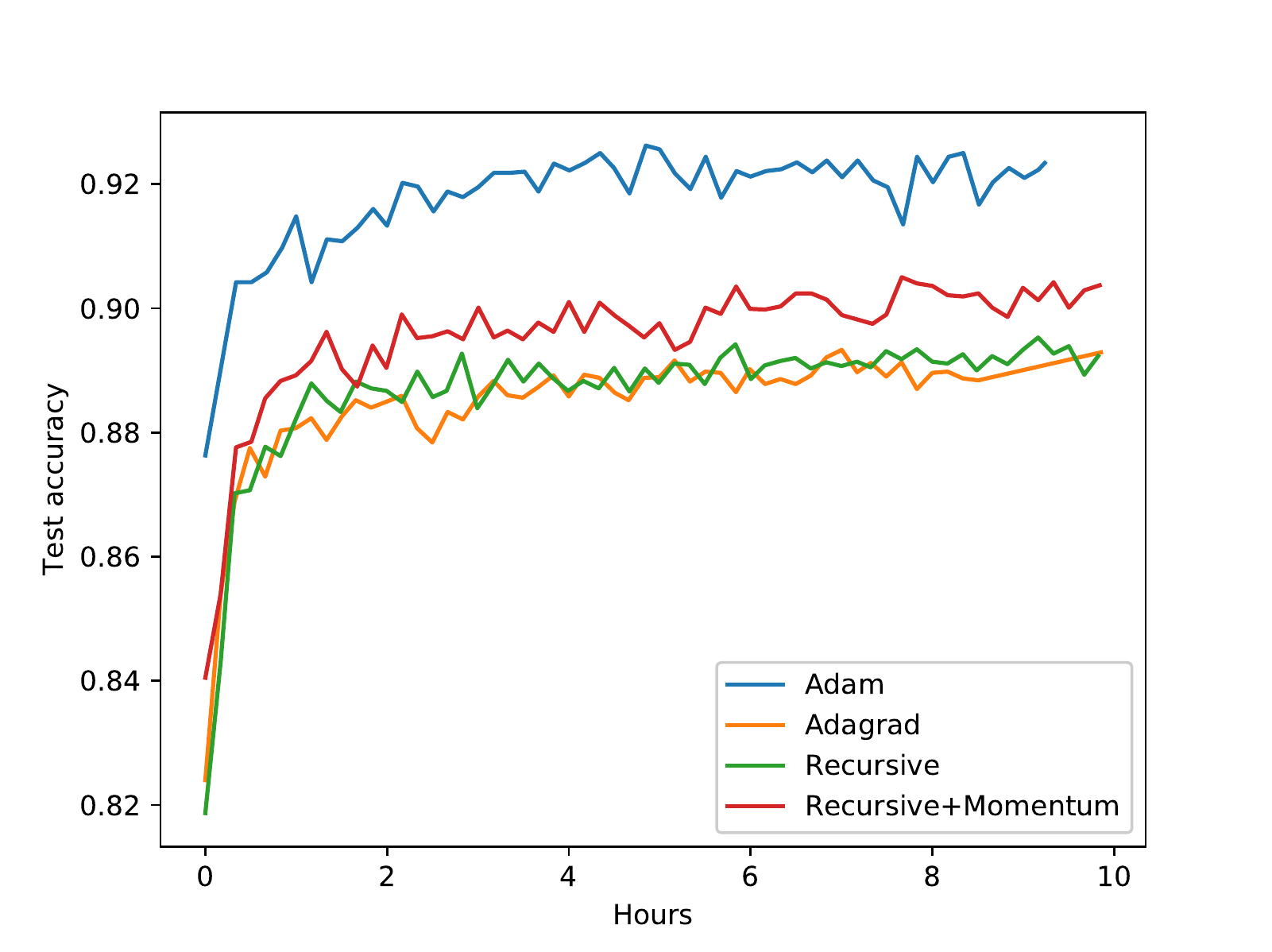}}\\
  \subfigure[][Train loss vs steps]{\includegraphics[width=.4\textwidth]{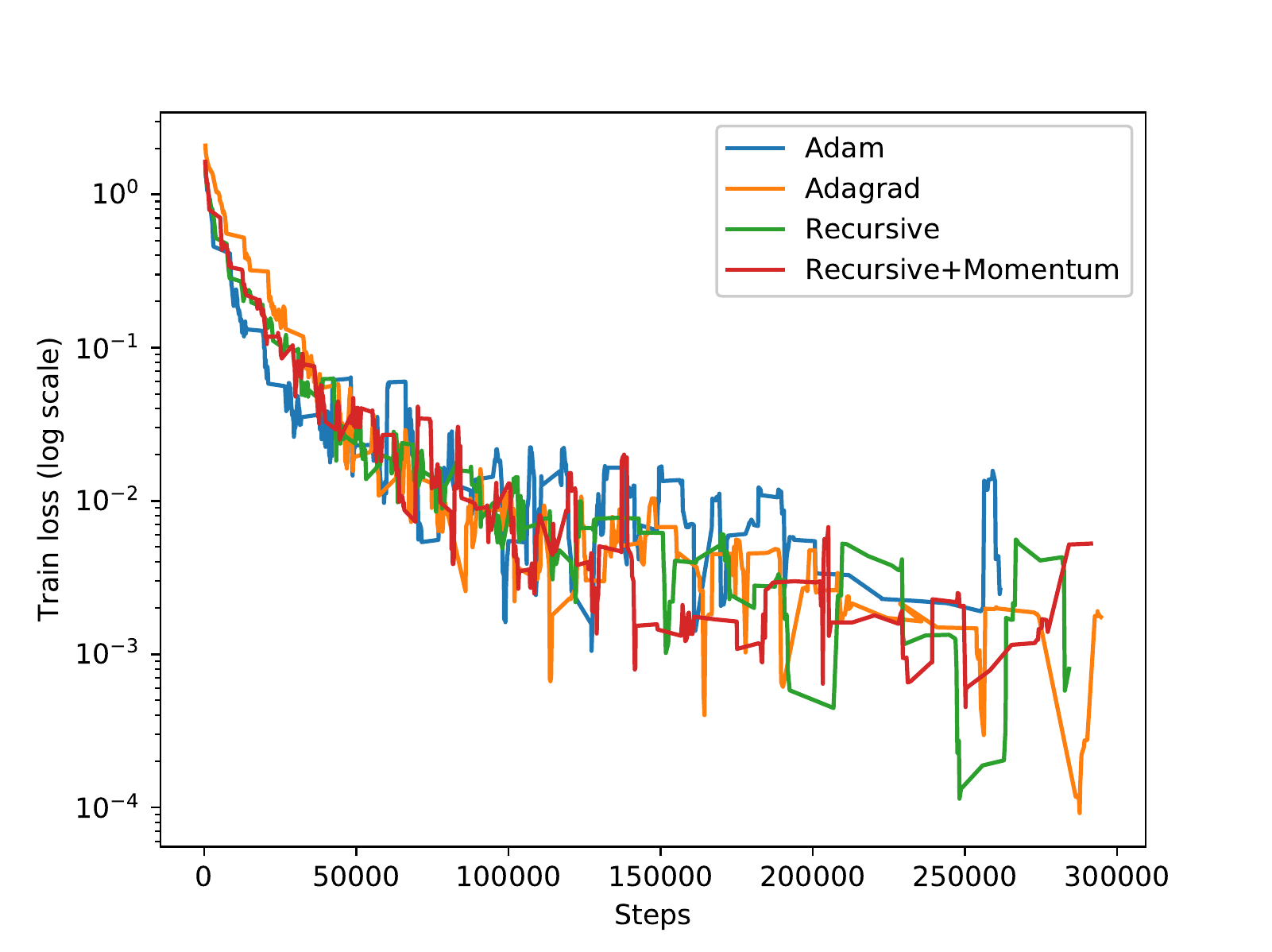}}\quad
  \subfigure[][Train loss vs time]{\includegraphics[width=.4\textwidth]{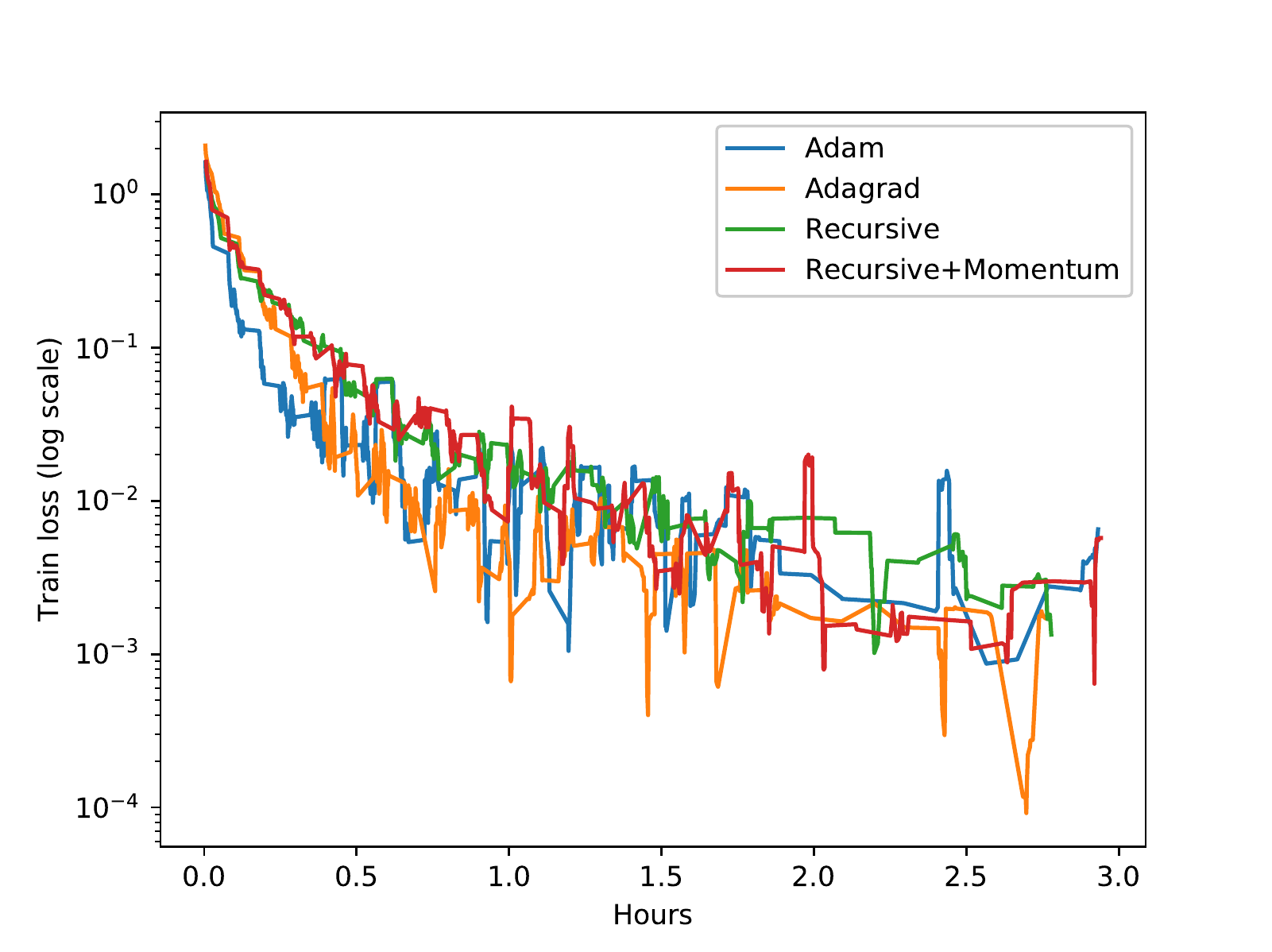}}
  \caption{CIFAR-10 with ResNet-32}
  \label{fig:image_cifar10}
\end{figure*}

\begin{figure*}
  \centering
  \subfigure[][Test accuracy vs steps]{\includegraphics[width=.4\textwidth]{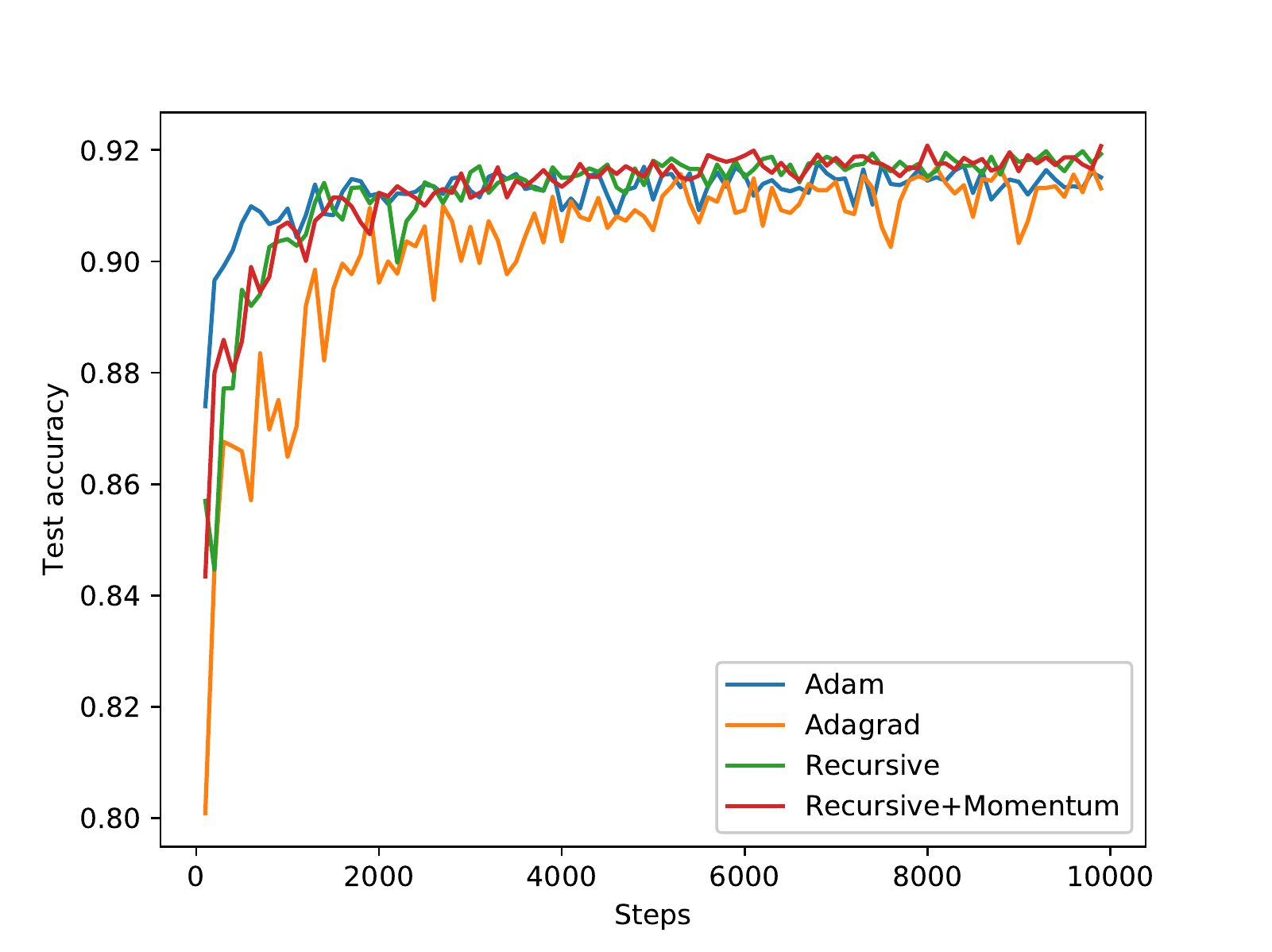}}\quad
  \subfigure[][Train loss vs steps]{\includegraphics[width=.4\textwidth]{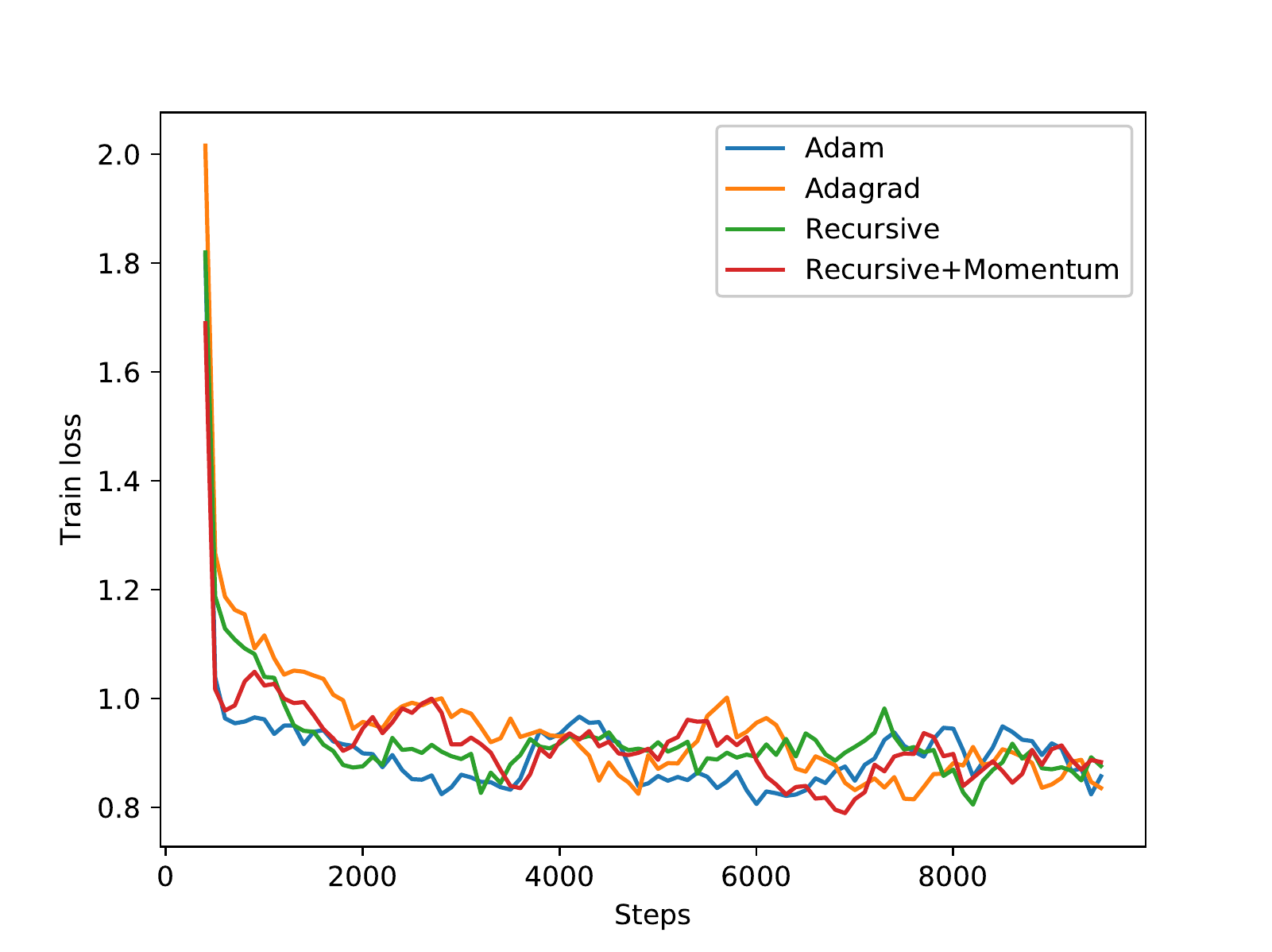}}
  \caption{MNIST with logistic regression}
  \label{fig:image_mnist-linear}
\end{figure*}

\begin{figure*}
  \centering
  \subfigure[][Test accuracy vs steps]{\includegraphics[width=.4\textwidth]{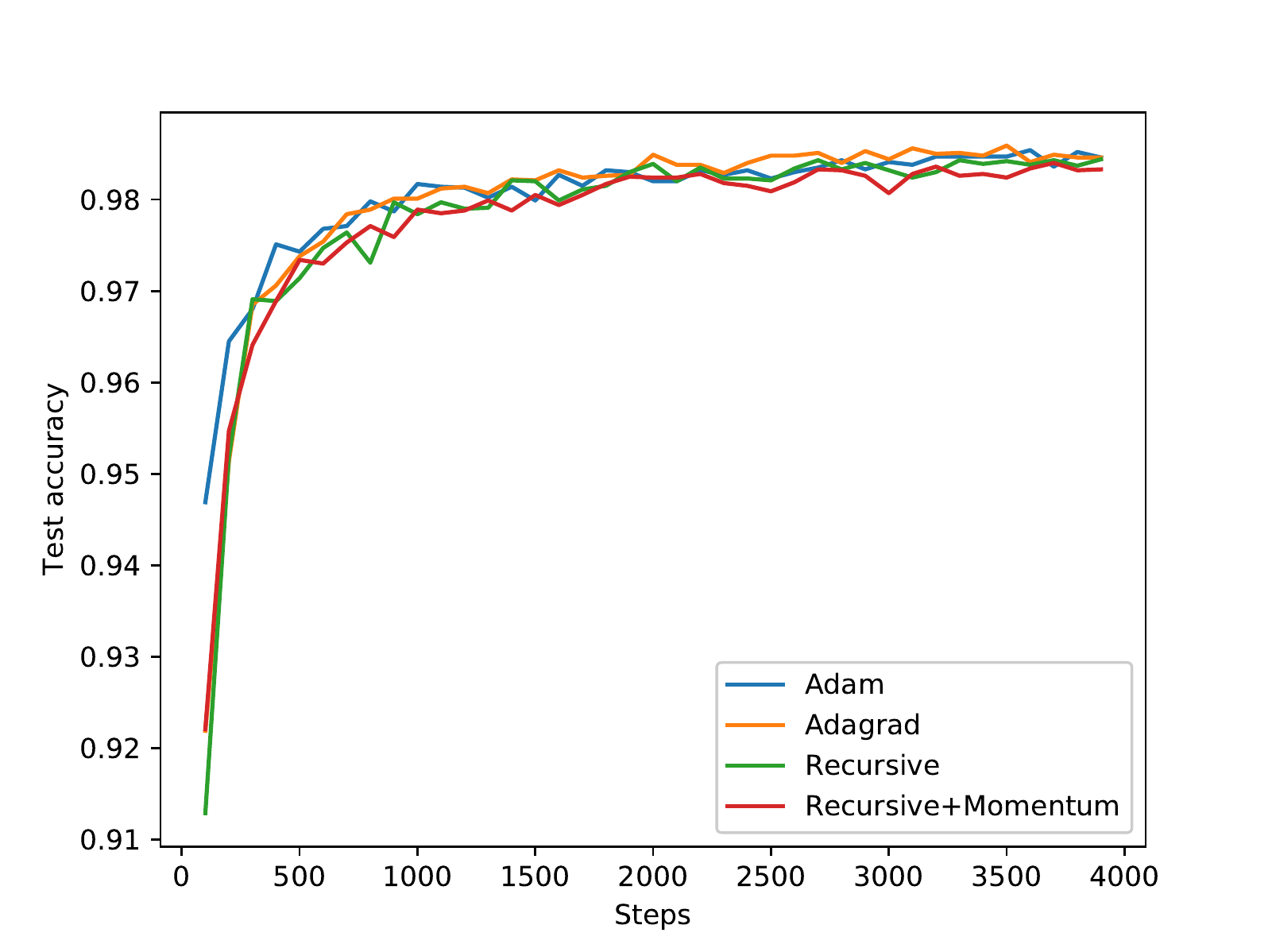}}\quad
  \subfigure[][Train loss vs steps]{\includegraphics[width=.4\textwidth]{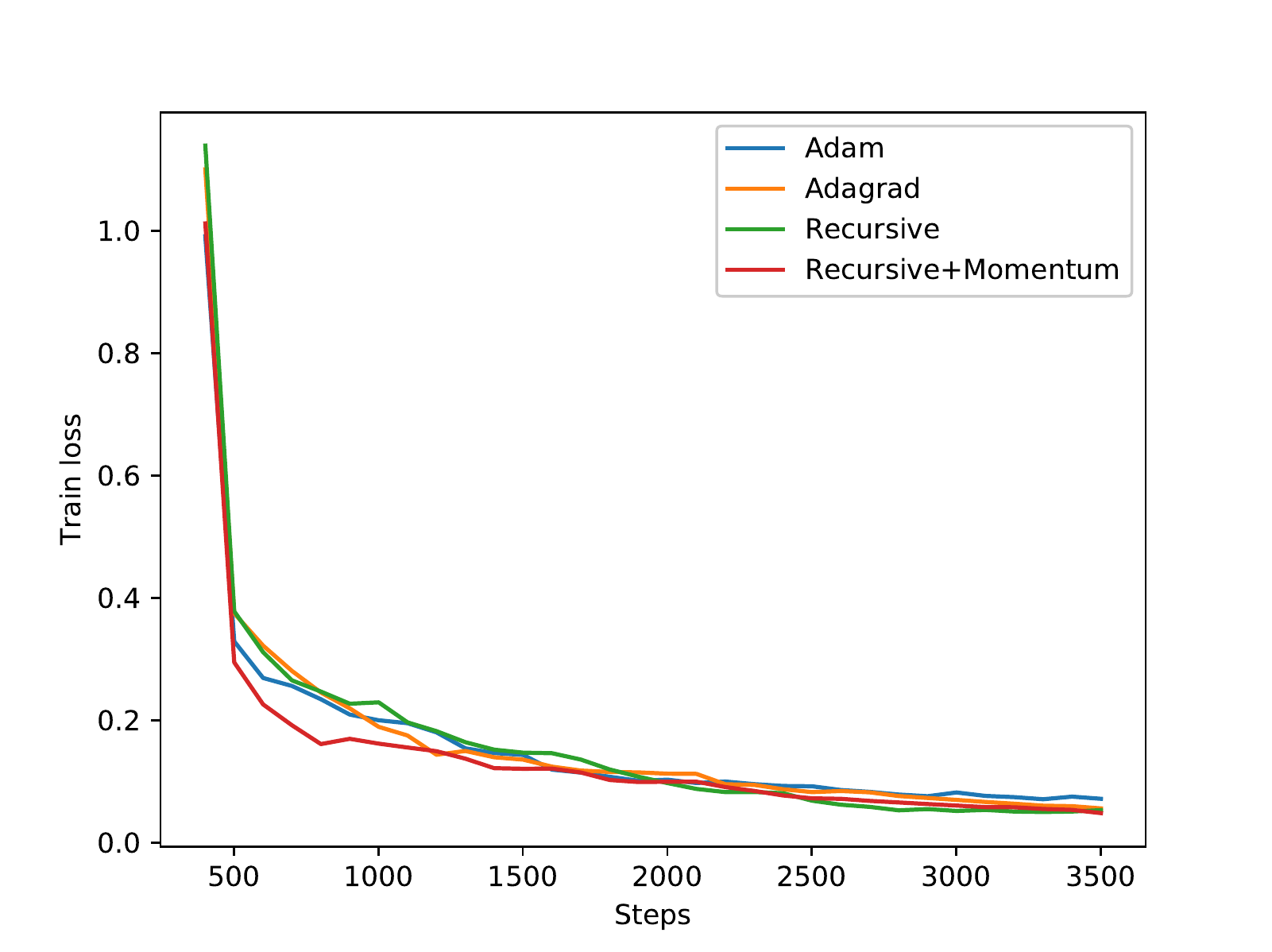}}
  \caption{MNIST with two layer fully connected network}
  \label{fig:image_mnist-fc}
\end{figure*}

\begin{figure*}
  \centering
    \subfigure[][Test accuracy vs steps]{\includegraphics[width=.4\textwidth]{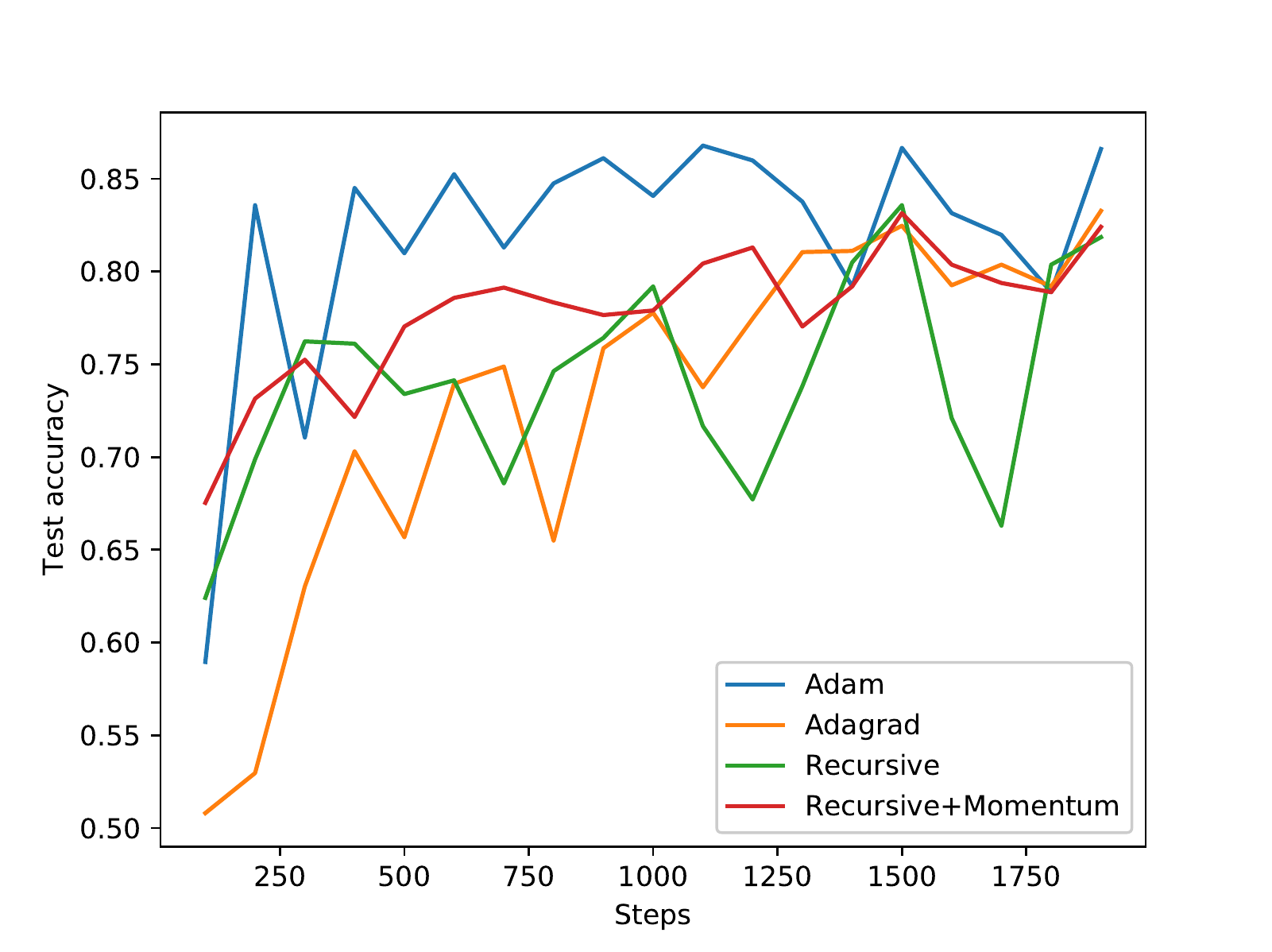}}\quad
  \subfigure[][Train loss vs steps]{\includegraphics[width=.4\textwidth]{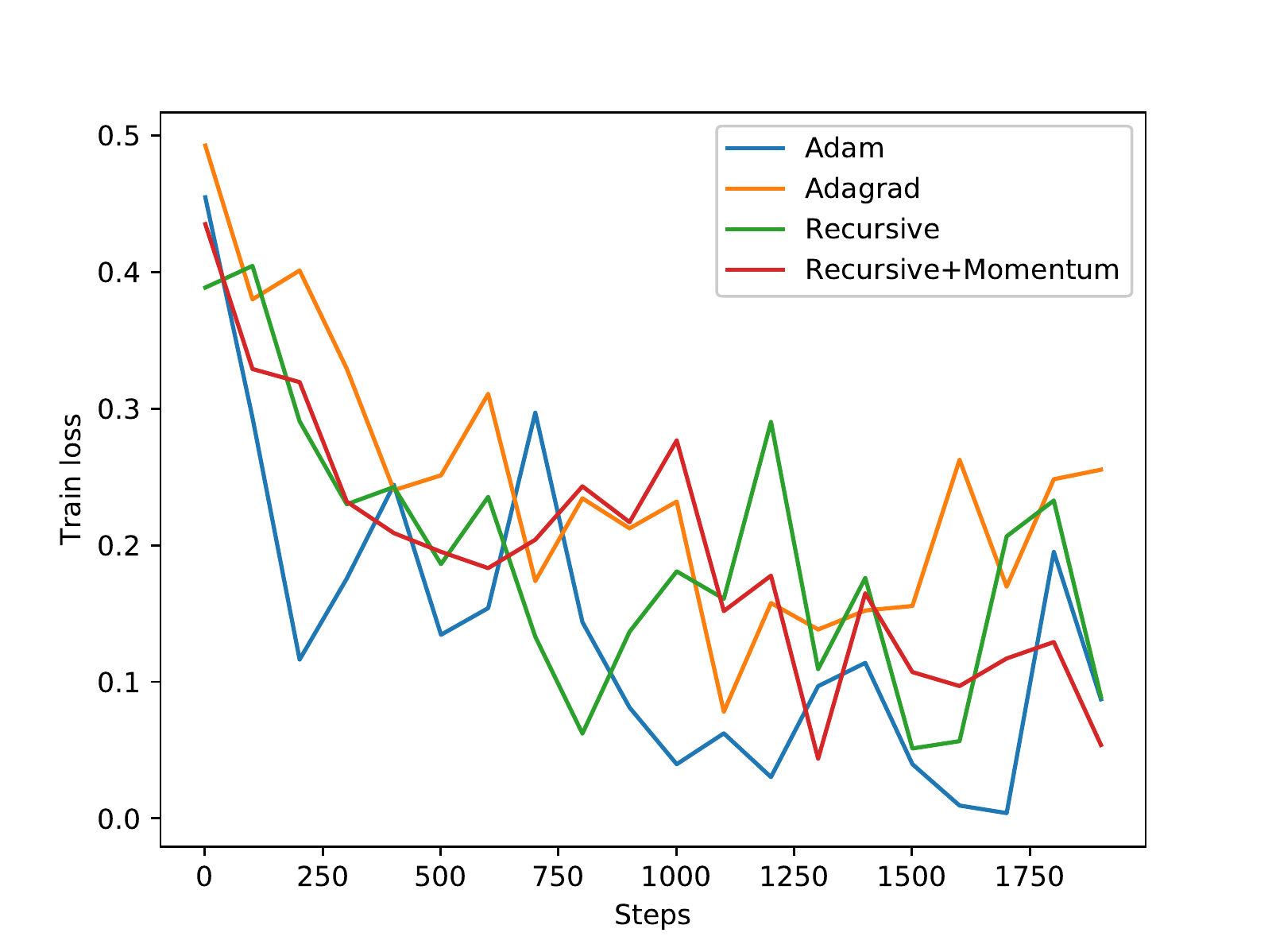}}
  \caption{IMDB sentiment classification with Transformer}
  \label{fig:sentiment_imdb}
\end{figure*}

\begin{figure*}
  \centering
  \subfigure[][Test accuracy vs steps]{\includegraphics[width=.4\textwidth]{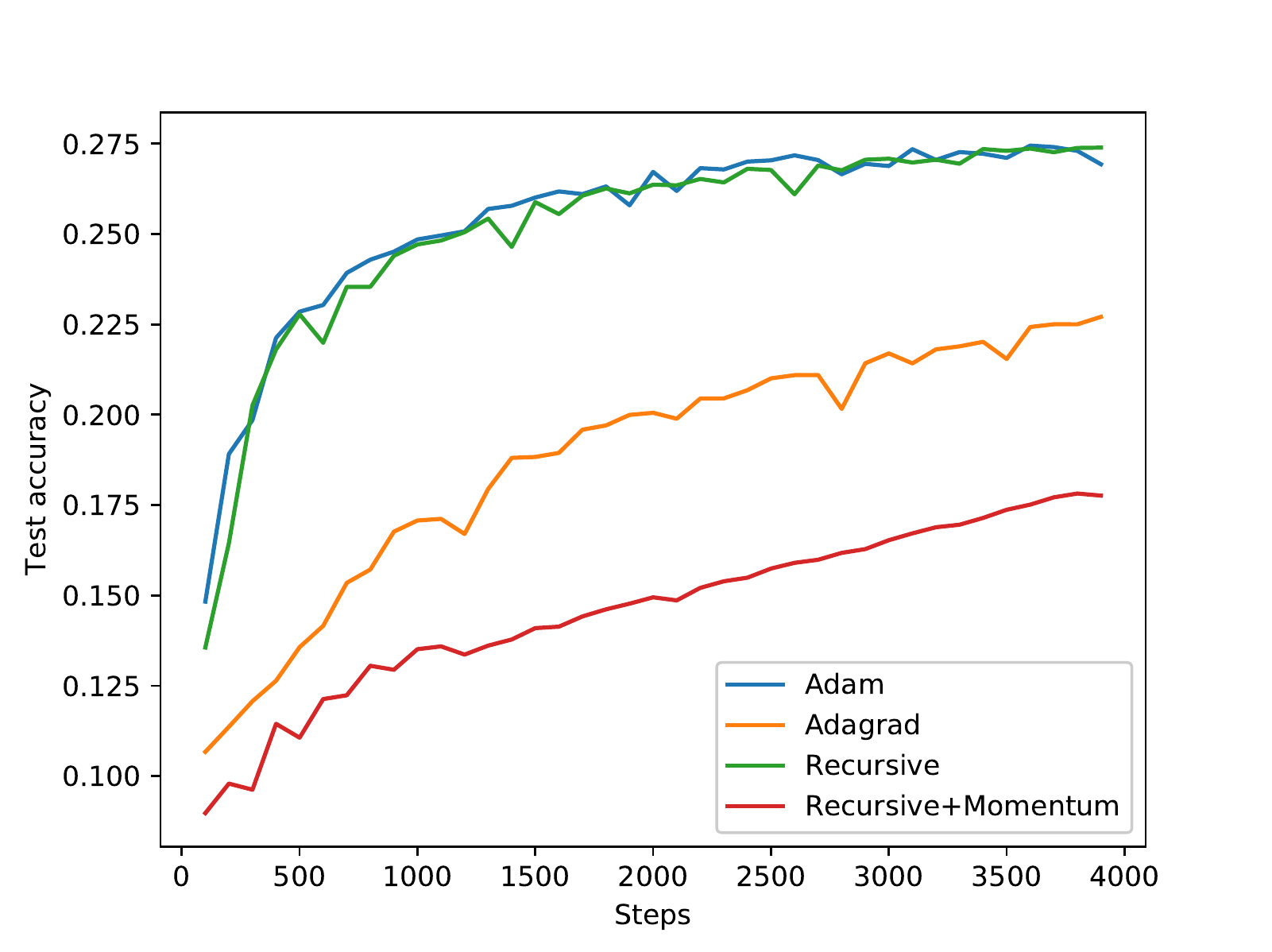}}\quad
  \subfigure[][Test accuracy vs time]{\includegraphics[width=.4\textwidth]{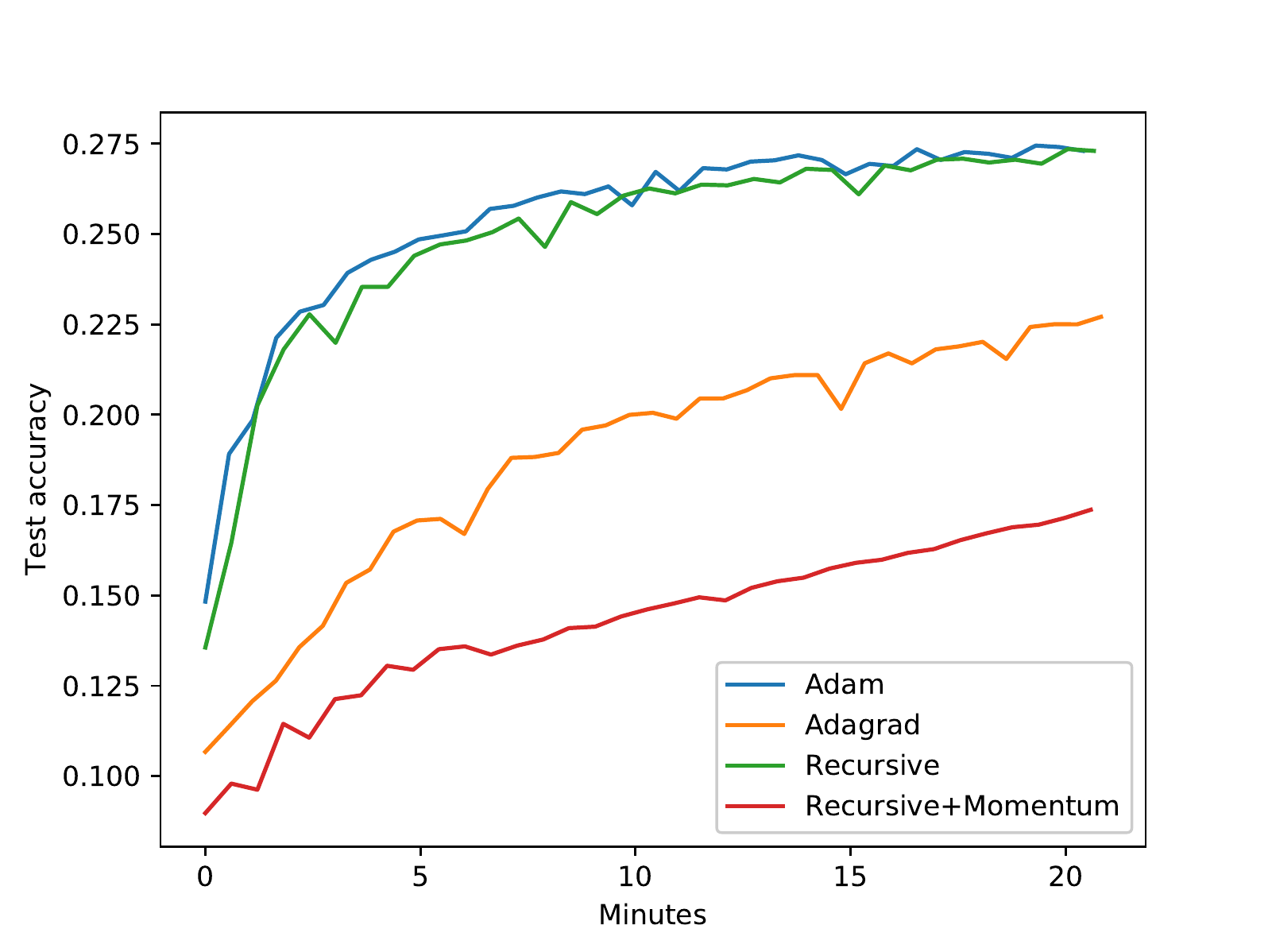}}\\
  \subfigure[][Test log perplexity vs steps]{\includegraphics[width=.4\textwidth]{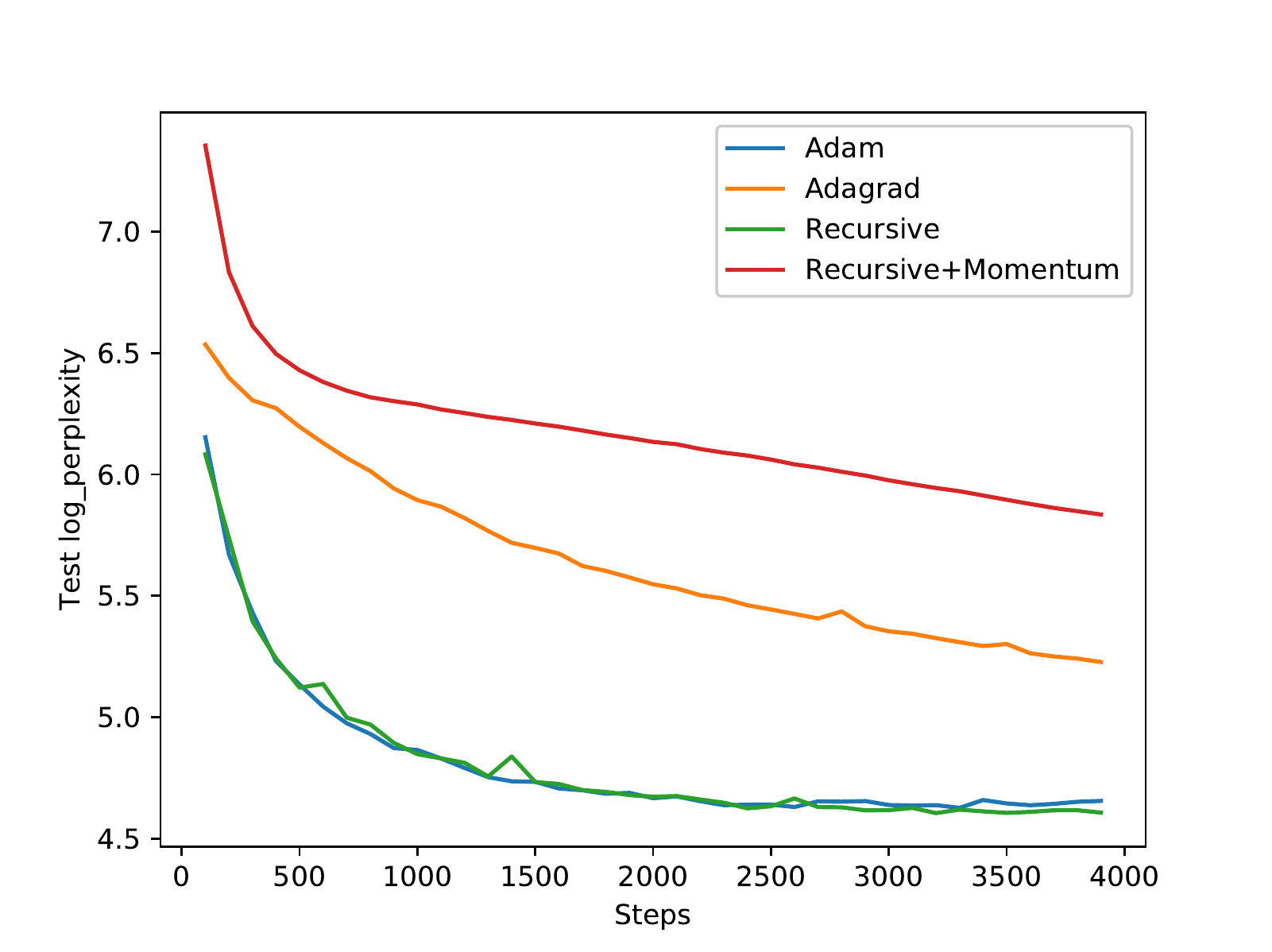}}\quad
  \subfigure[][Test log perplexity vs time]{\includegraphics[width=.4\textwidth]{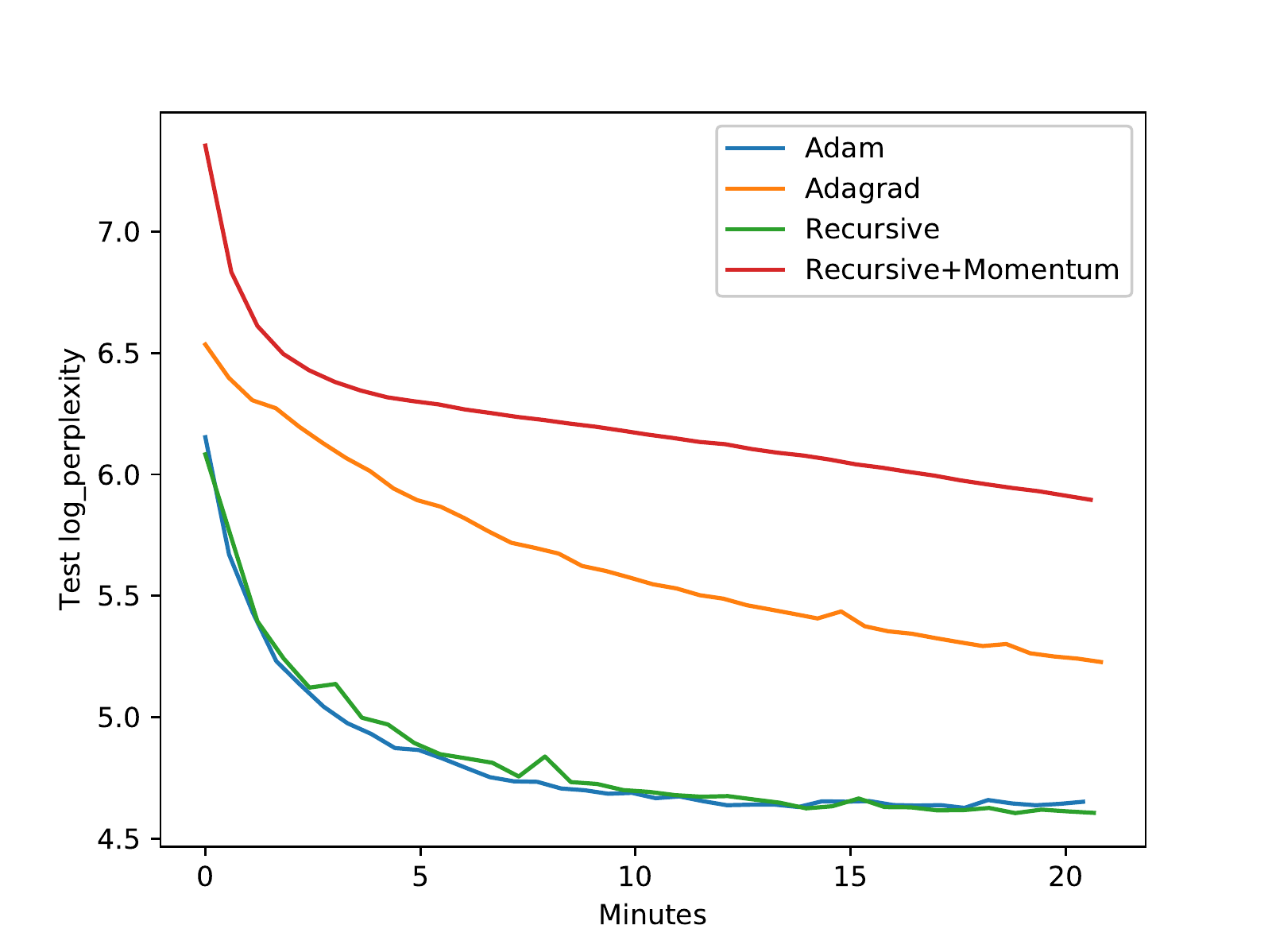}}\\
  \subfigure[][Train loss vs steps]{\includegraphics[width=.4\textwidth]{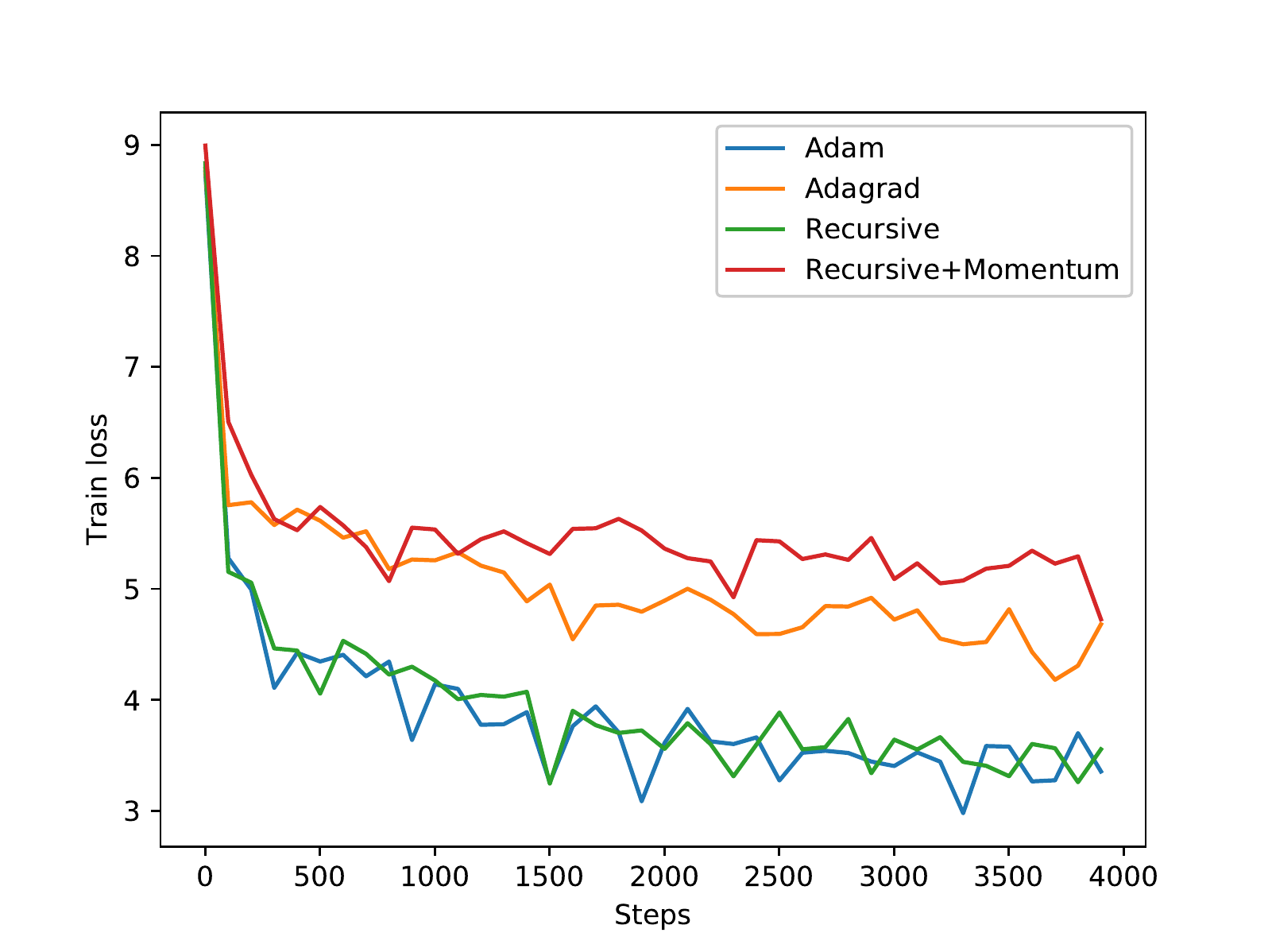}}\quad
  \subfigure[][Train loss vs time]{\includegraphics[width=.4\textwidth]{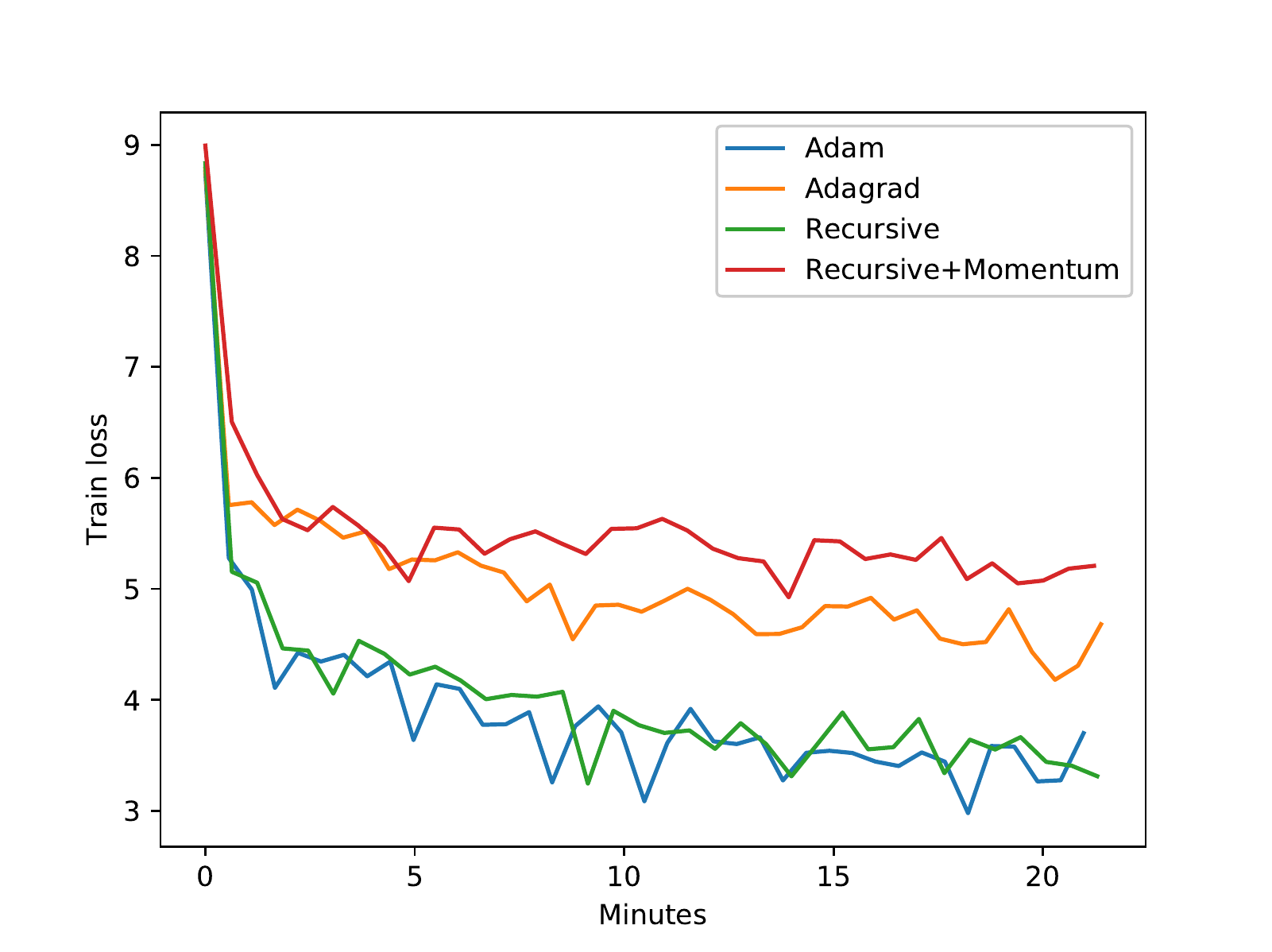}}
  \caption{Penn Tree Bank with Transformer}
  \label{fig:languagemodel_ptb10k}
\end{figure*}

\begin{figure*}
  \centering
  \subfigure[][Test accuracy vs steps]{\includegraphics[width=.4\textwidth]{figs/languagemodel_lm1b32k/Test-accuracy-Steps-eps-converted-to.pdf}}\quad
  \subfigure[][Test accuracy vs time]{\includegraphics[width=.4\textwidth]{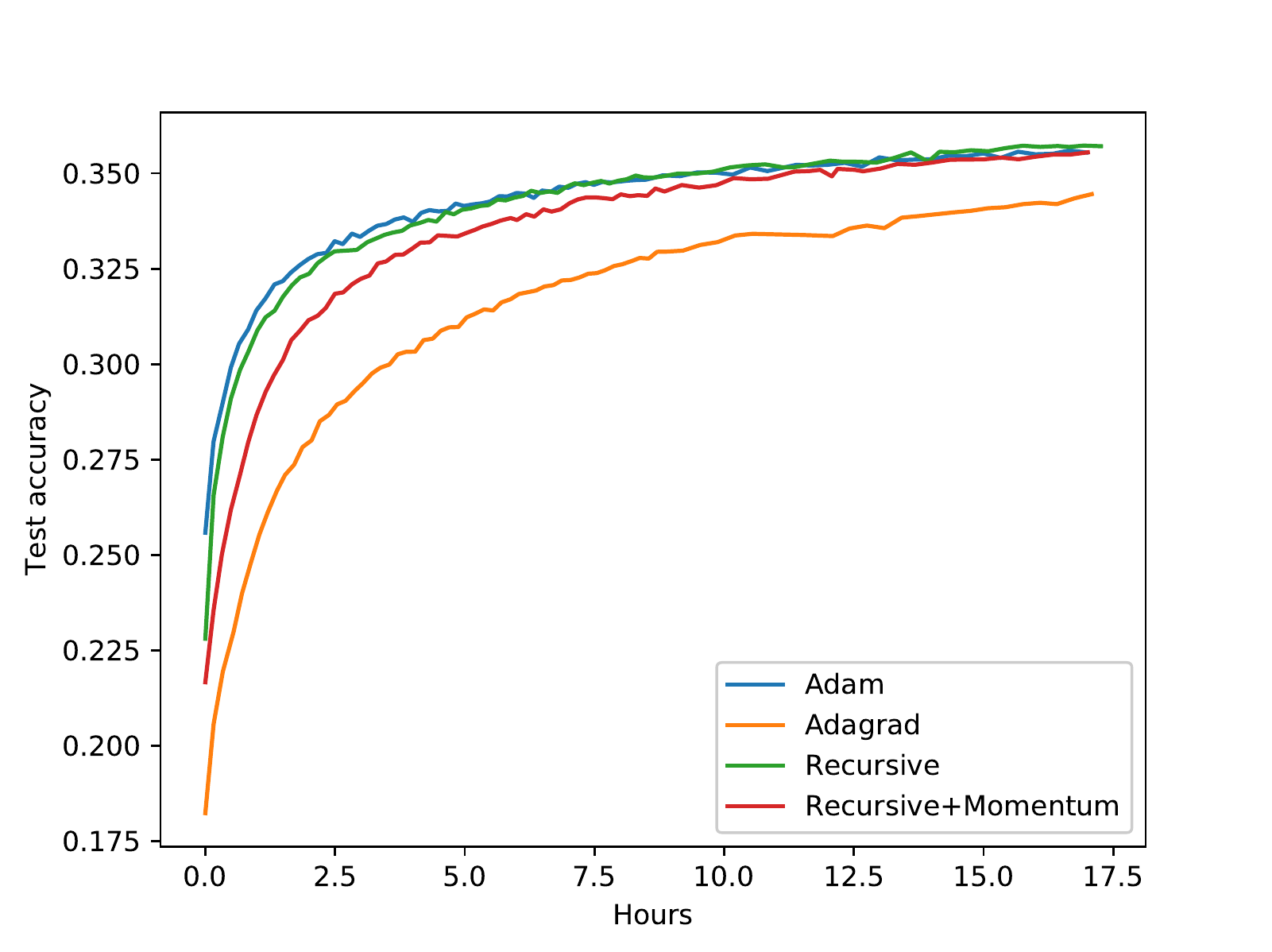}}\\
  \subfigure[][Test log perplexity vs steps]{\includegraphics[width=.4\textwidth]{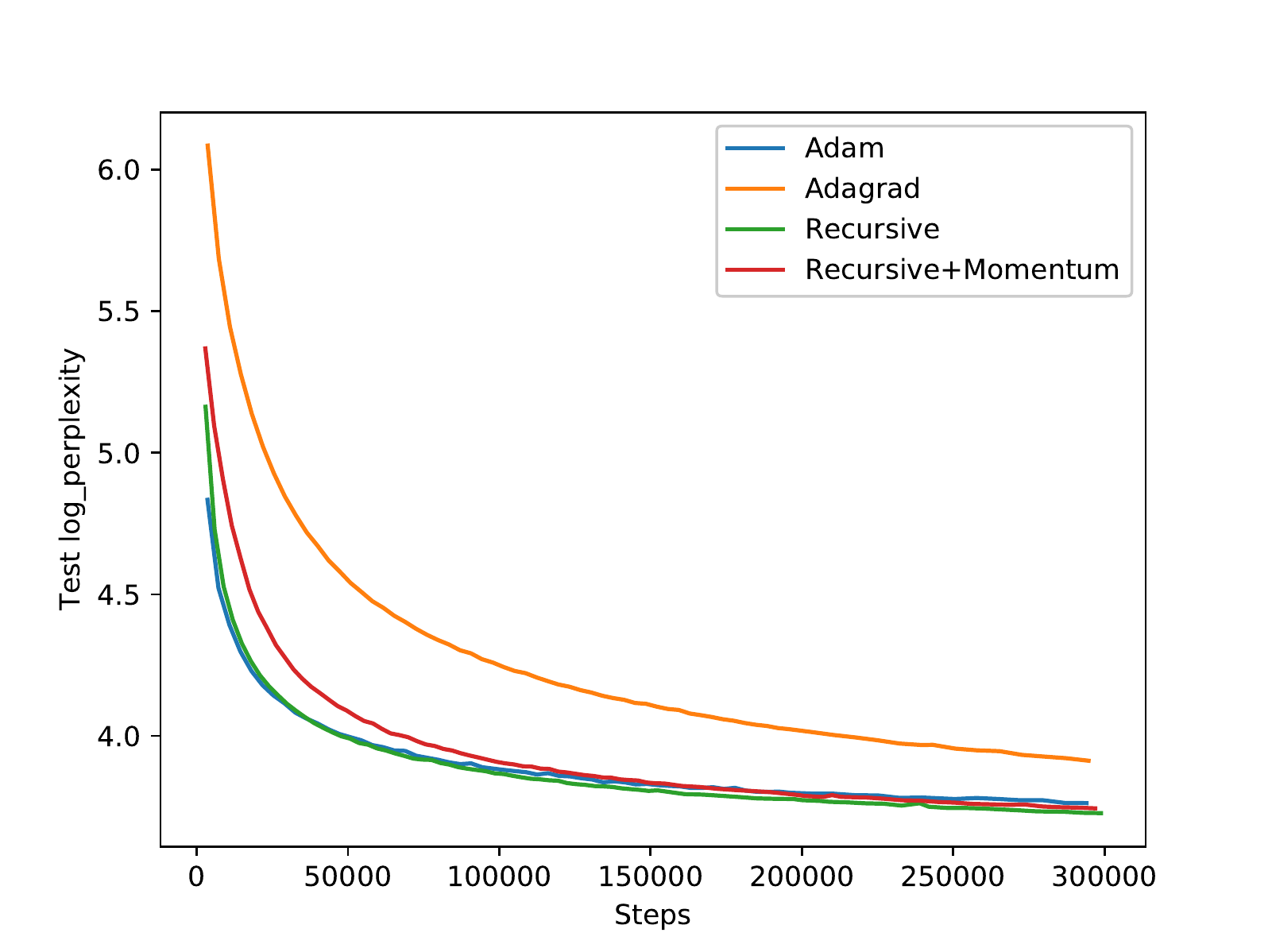}}\quad
  \subfigure[][Test log perplexity vs time]{\includegraphics[width=.4\textwidth]{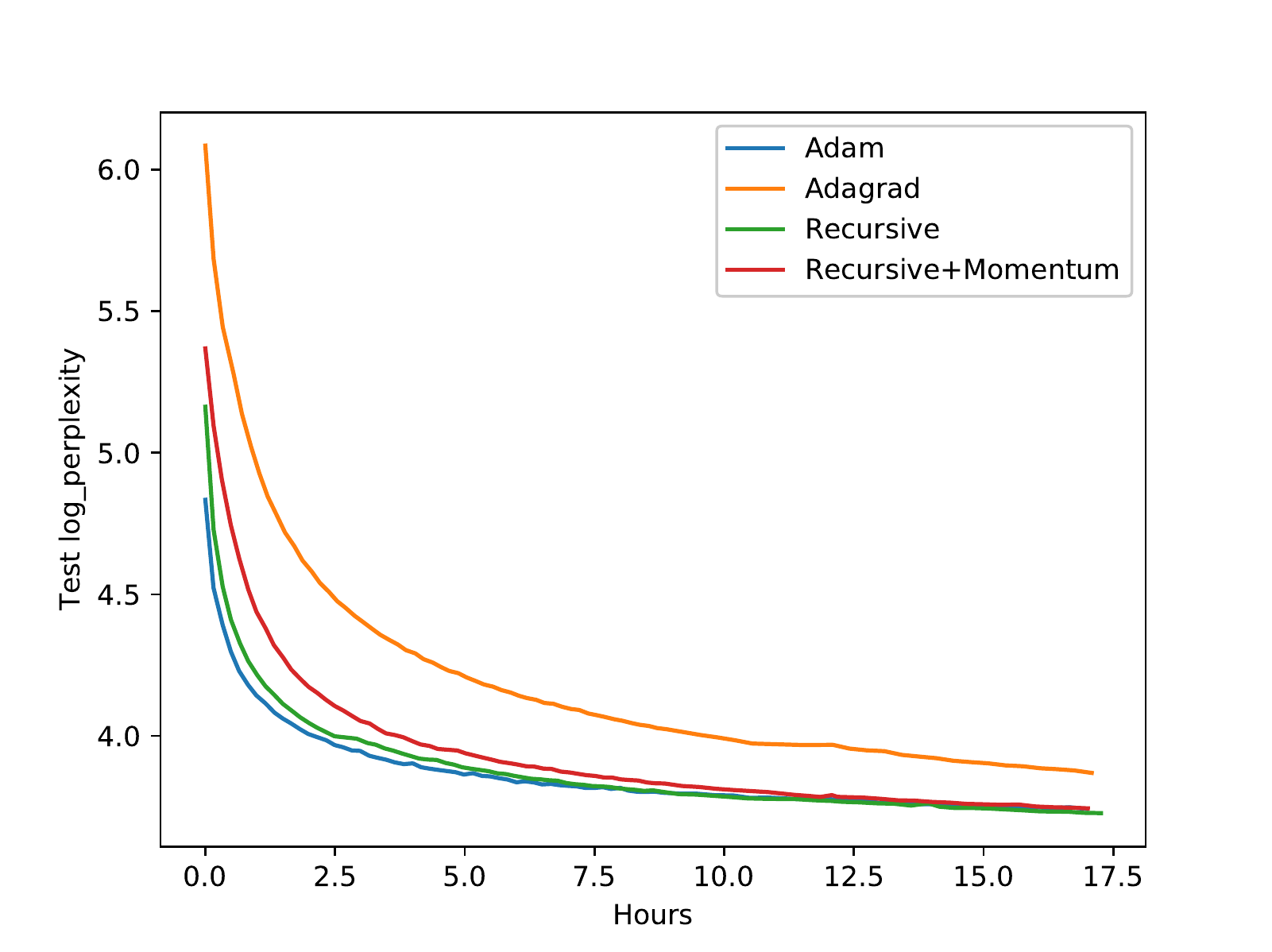}}\\
  \subfigure[][Train loss vs steps]{\includegraphics[width=.4\textwidth]{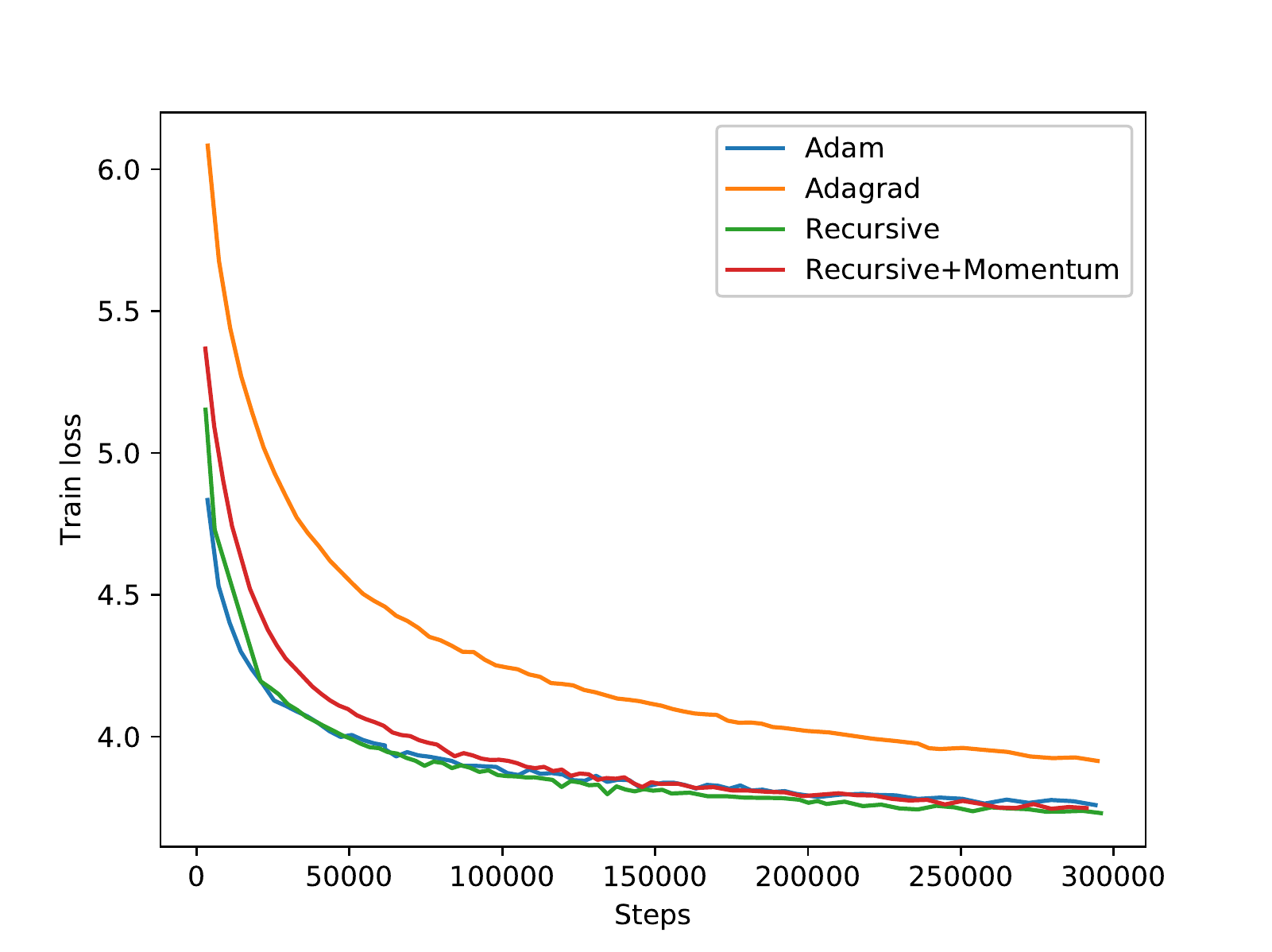}}\quad
  \subfigure[][Train loss vs time]{\includegraphics[width=.4\textwidth]{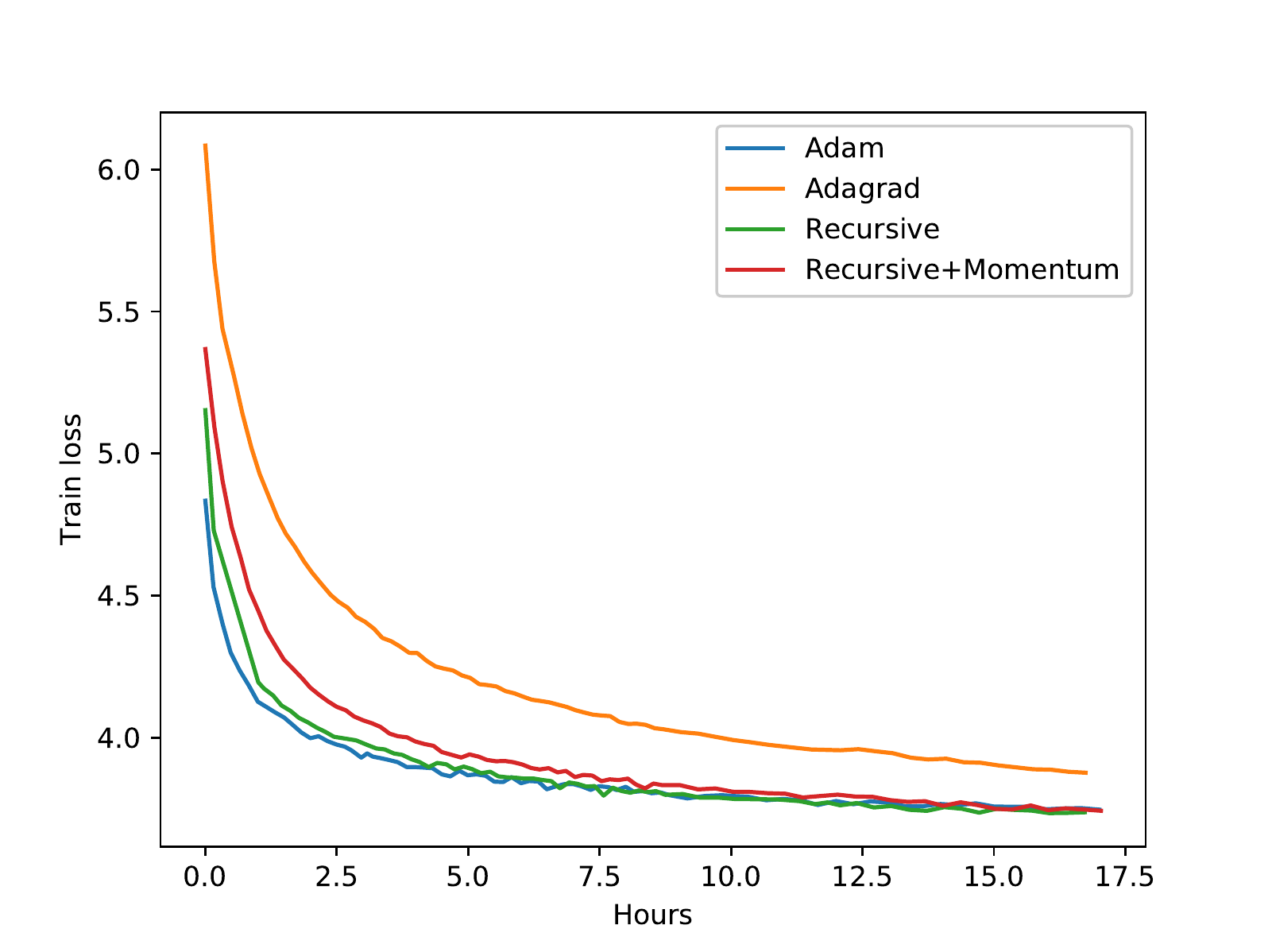}}
  \caption{LM1B with Transformer}
  \label{fig:languagemodel_lm1b32k}
\end{figure*}

\begin{figure*}
  \centering
  \subfigure[][CIFAR-10 with ResNe-32t]{\includegraphics[width=.4\textwidth]{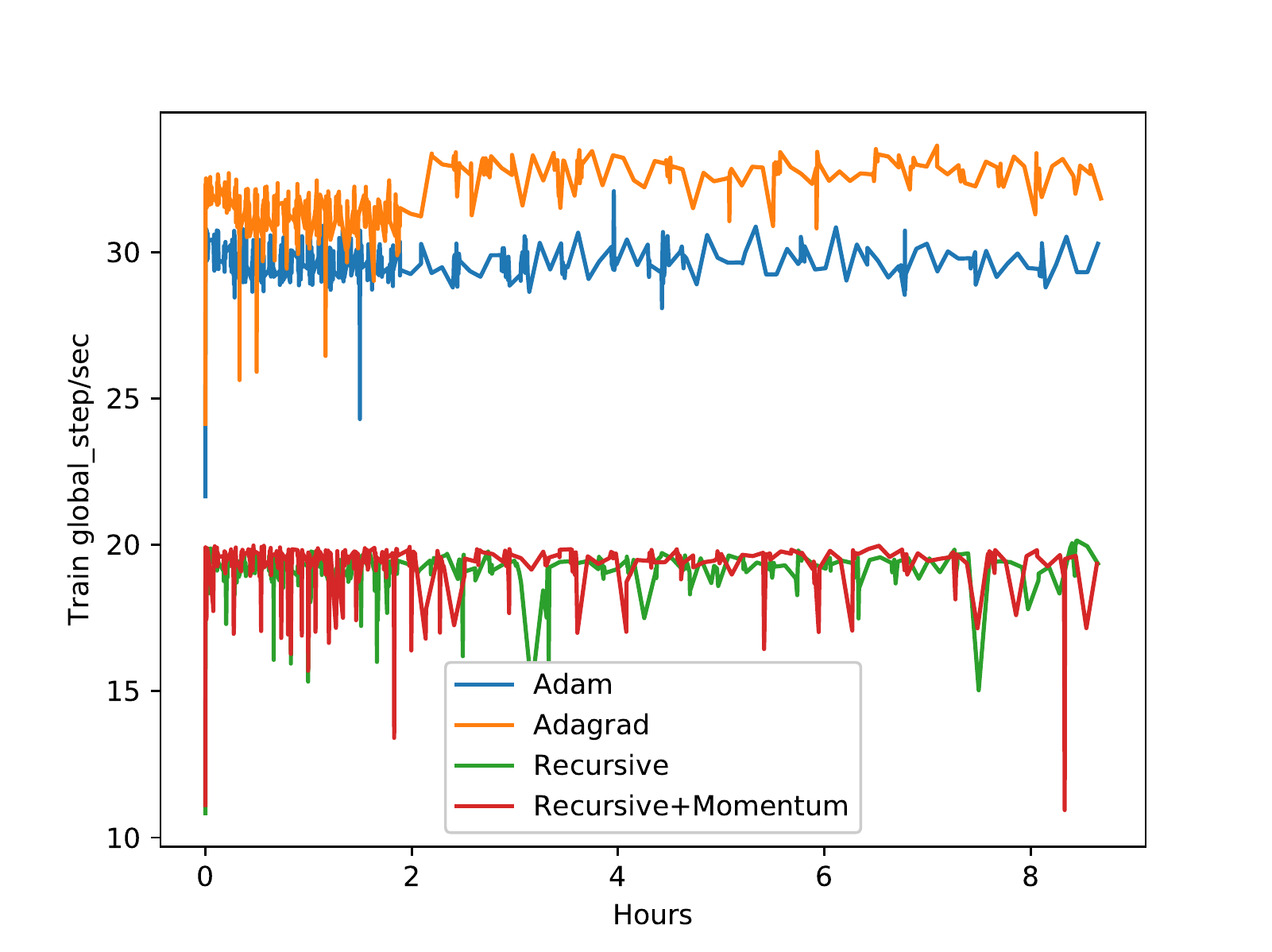}}\quad
  \subfigure[][IMBD with Transformer]{\includegraphics[width=.4\textwidth]{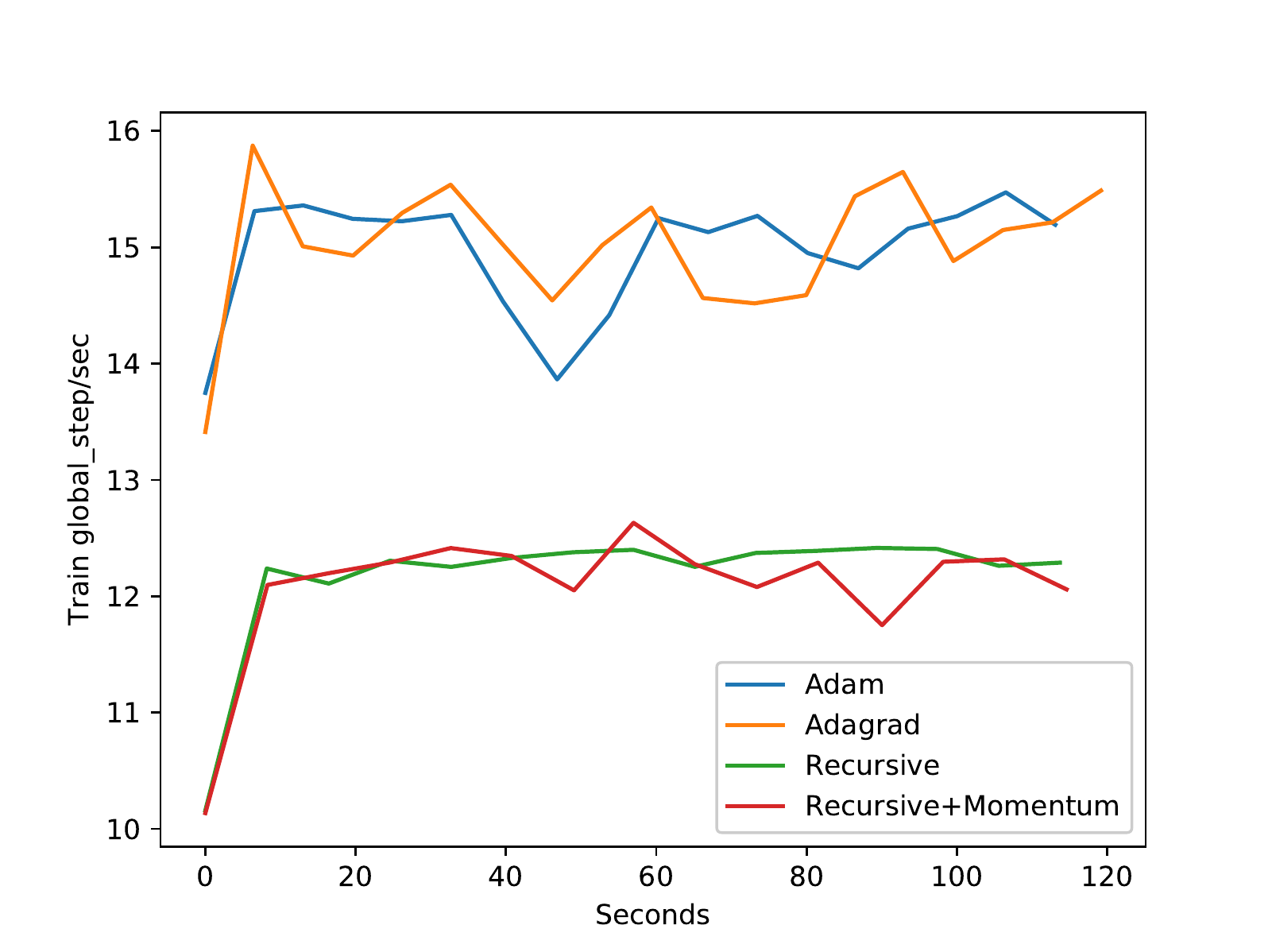}}\\
  \subfigure[][MNIST with logistic regression]{\includegraphics[width=.4\textwidth]{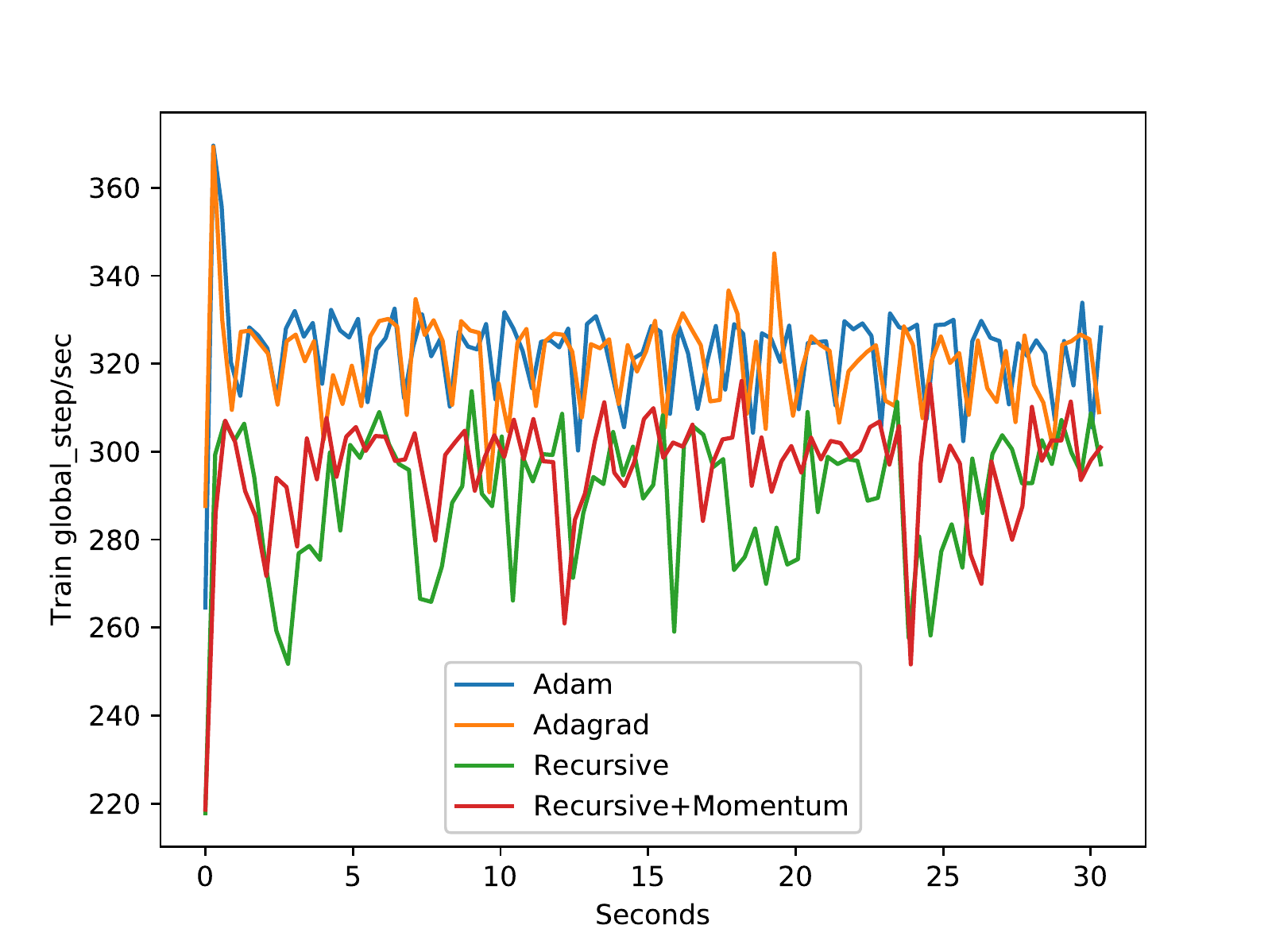}}\quad
  \subfigure[][MNIST with fully connected network]{\includegraphics[width=.4\textwidth]{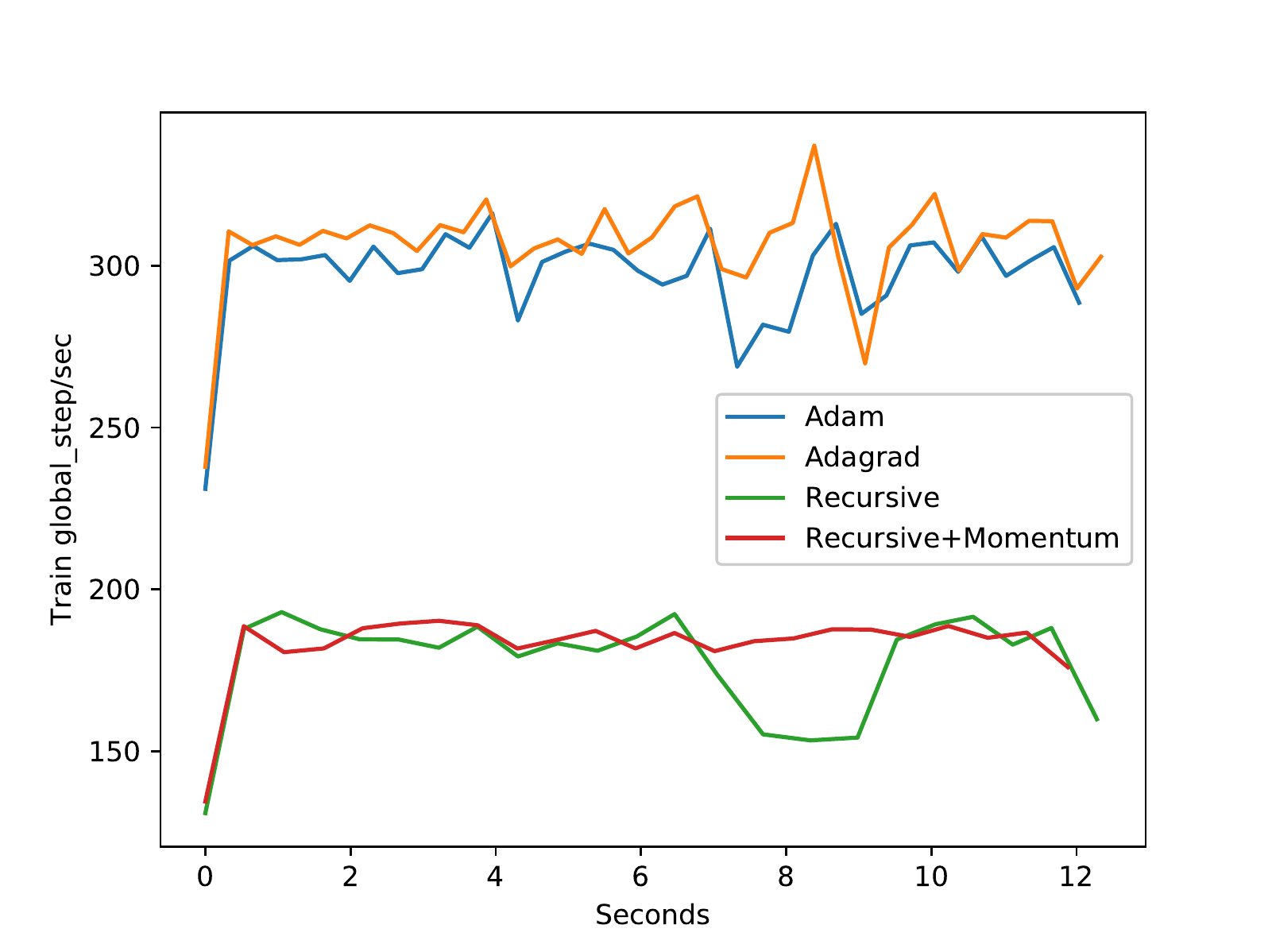}}\\
  \subfigure[][Penn Tree Bank with Transformer]{\includegraphics[width=.4\textwidth]{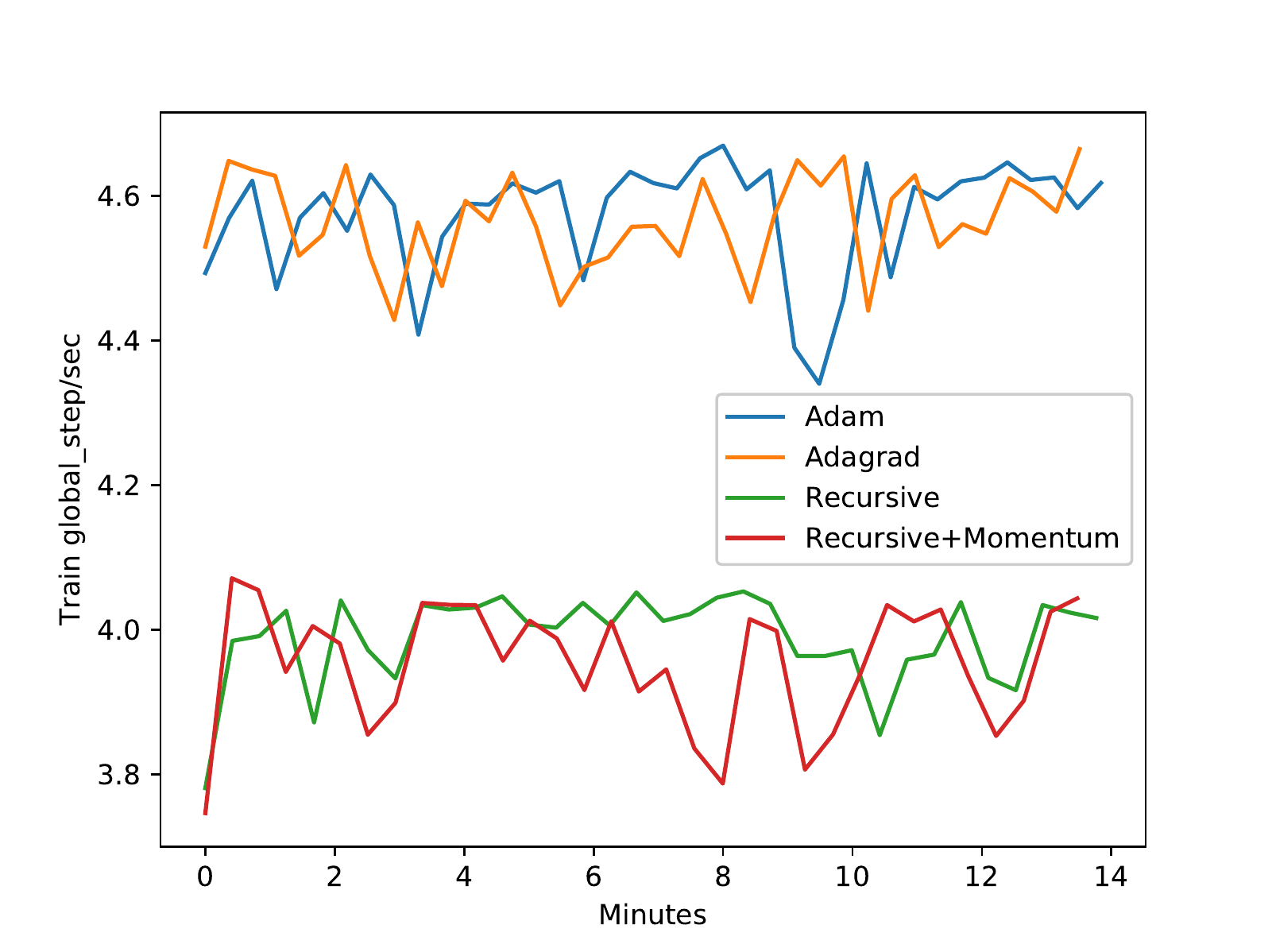}}\quad
  \subfigure[][LM1B with Transformer]{\includegraphics[width=.4\textwidth]{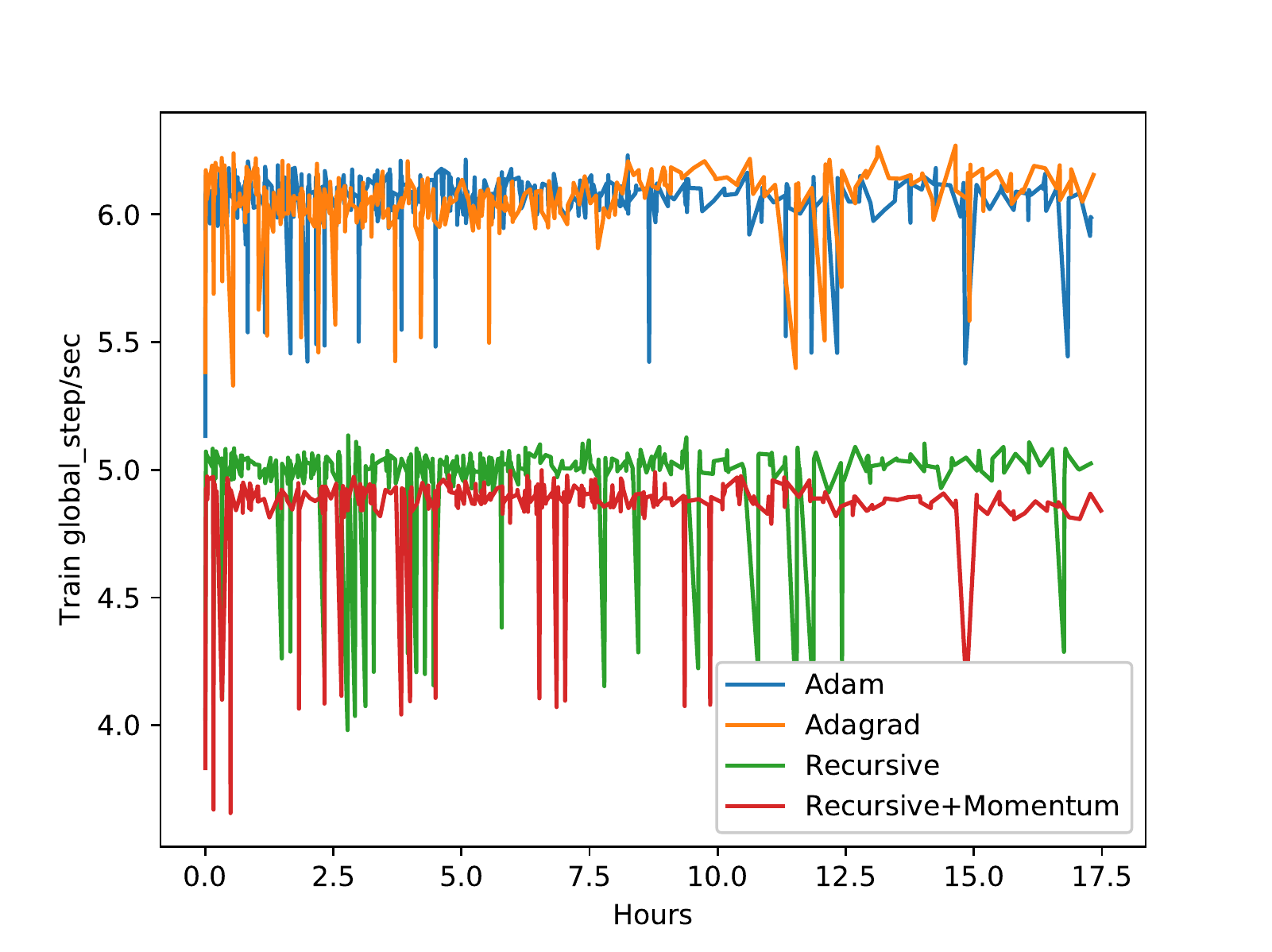}}
  \caption{Number of iterations per second}
  \label{fig:speed}
\end{figure*}
\end{document}